\documentclass[journal,10pt]{IEEEtran}
\IEEEoverridecommandlockouts
\usepackage{cite}
\usepackage{amsmath,amssymb,amsfonts,amsthm}
\usepackage{algorithmic}
\usepackage{graphicx}
\usepackage{textcomp}
\usepackage{multirow}
\usepackage{subfigure}
\usepackage{xcolor, algorithm, booktabs}
\usepackage{bm}
\usepackage{comment}
\usepackage{units}
\usepackage{makecell}
\usepackage{wasysym}
\usepackage{balance}

\def\BibTeX{{\rm B\kern-.05em{\sc i\kern-.025em b}\kern-.08em
    T\kern-.1667em\lower.7ex\hbox{E}\kern-.125emX}}

\newcommand{\obs}{\text{obs}}

\newcommand{\HD}{\text{HD}}
\newcommand{\old}{\text{old}}

\allowdisplaybreaks
\newtheorem{theorem}{Theorem}
\def\BibTeX{{\rm B\kern-.05em{\sc i\kern-.025em b}\kern-.08em
    T\kern-.1667em\lower.7ex\hbox{E}\kern-.125emX}}
\begin{document}

\title{Decentralized Consensus Inference-based Hierarchical Reinforcement Learning for Multi-Constrained UAV Pursuit-Evasion Game
\\

}
\author{Yuming Xiang, Sizhao Li, Rongpeng Li, Zhifeng Zhao, and Honggang Zhang

\thanks{
    *This work 
    has been accepted by IEEE Transactions on Neural Networks and Learning Systems in June 2025, and is supported in part by the National Key R\&D Program of China under Grant 2024YFE0200600, and the Zhejiang Provincial Natural Science Foundation of China under Grant LR23F010005.
    
    Yuming Xiang and Sizhao Li contribute equally.
    Yuming Xiang and Sizhao Li and Rongpeng Li are with the College of Information Science and Electronic Engineering, Zhejiang University, Hangzhou 310027, China (e-mail: \{xiangym1999; liszh5; lirongpeng\}@zju.edu.cn).
   
   Zhifeng Zhao is with Zhejiang Lab, Hangzhou, China as well as the College of Information Science and Electronic Engineering, Zhejiang University, Hangzhou 310027, China (e-mail: zhaozf@zhejianglab.com).

   Honggang Zhang is with City University of Macau, Macau, China (email: hgzhang@cityu.edu.mo).
  }
\vspace{-20pt}
}

\maketitle

\begin{abstract}
Multiple quadrotor unmanned aerial vehicle (UAV) systems have garnered widespread research interest and fostered tremendous interesting applications, especially in multi-constrained pursuit-evasion games (MC-PEG).
The Cooperative Evasion and Formation Coverage (CEFC) task, where the UAV swarm aims to maximize formation coverage across multiple target zones while collaboratively evading predators, belongs to one of the most challenging issues in MC-PEG, especially under communication-limited constraints.
This multifaceted problem, which intertwines responses to obstacles, adversaries, target zones, and formation dynamics, brings up significant high-dimensional complications in locating a solution.
In this paper, we propose a novel two-level framework (i.e., Consensus Inference-based Hierarchical Reinforcement Learning (CI-HRL)), which delegates target localization to a high-level policy, while adopting a low-level policy to manage obstacle avoidance, navigation, and formation.
Specifically, in the high-level policy, we develop a novel multi-agent reinforcement learning module, Consensus-oriented Multi-Agent Communication (ConsMAC), to enable agents to perceive global information and establish consensus from local states by effectively aggregating neighbor messages.
Meanwhile, we leverage an Alternative Training-based Multi-agent proximal policy optimization (AT-M) and policy distillation to accomplish the low-level control.
The experimental results, including the high-fidelity software-in-the-loop (SITL) simulations, validate that CI-HRL provides a superior solution with enhanced swarm's collaborative evasion and task completion capabilities.
\end{abstract}


\section{Introduction}
Nowadays, quadrotor Unmanned Aerial Vehicles (UAVs) have 
demonstrated great potential in costly or human-unfriendly tasks (e.g., disaster response \cite{Marwan2024Meta}), 
due to their agility, cost-effectiveness, and compact size.
Nevertheless, the UAV swarm is likely to be exposed to an adversarial environment, where a hostile factor or agent might attack the affiliated members, and must respond promptly to boost the survival opportunity. 
Typically, such a Cooperative Evasion and Formation Coverage (CEFC) scenario is formulated as a Multi-Constrained Pursuit-Evasion Game (MC-PEG) \cite{Pham2017Distributed},
wherein preys (i.e., UAVs) shall maximize the ratio of accomplishing planned missions while avoiding attacks from predators (i.e., the hostile attacker)\cite{hu2024transfer}.
However, as the complexity of MC-PEG or environmental constraints escalates, conventional control algorithms face several aspects of prominent shortcomings, such as oversimplified predator strategies (e.g., fixed trajectory)\cite{Young2020Consensus}, lack of inter-group cooperation\cite{Vibhav2022Multi}, limitation to one single formation pattern \cite{konda2020decentralized}, and unrealistic presumption on the availability of global information\cite{Zhang2023RealCity,yuan2025multiagent}.

\begin{table}[t]
    \vspace{-0em}
    \centering
    \caption{The key abbreviation list.}
    \label{tab:keyabbreviation}
    \vspace{-.3em}
    \begin{tabular}{c|m{6.1cm}}
    \toprule
    \textbf{Abbreviation} & \textbf{Full Name} \\
    \midrule
     AT-M & Alternative Training-based Multi-agent proximal policy optimization \\
     CEFC & Cooperative Evasion and Formation Coverage \\
     CI-HRL & Consensus Inference-based Hierarchical Reinforcement Learning \\
     ConsMAC & Consensus-oriented Multi-Agent Communication \\
     MC-PEG &  Multi-Constrained Pursuit-Evasion Games \\
     SITL & Software-In-The-Loop \\

    \bottomrule
    \end{tabular}
    \vspace{-1.8em}
\end{table}

Benefiting from the robust adaptability in complex environments, Multi-Agent Reinforcement Learning (MARL)-based approaches\cite{lowe2017multi, yan2023collision, chai2023nvif} have been widely adopted. 
For instance, QMIX\cite{rashid2020monotonic} and its variants\cite{guan2022efficient} have yielded appealing results in complex multi-agent scenarios, and HASAC\cite{liu2024maximum}, as the latest state-of-the-art (SOTA) algorithm,
has also demonstrated outstanding capabilities in many benchmarks.
While these advancements are noteworthy, recent MARL implementations in practical systems still heavily rely on Multi-Agent Proximal Policy Optimization (MAPPO) \cite{schulman2017proximal,yu2022surprising}  due to its stability and computational efficiency\cite{  xie2024multi}.
However, existing MAPPO-based multi-robot control frameworks \cite{yuan2025multiagent} typically adopt oversimplified assumptions (e.g., global obstacle coordinates, small-scale scenarios), limiting their applicability to real-world environments.
Nevertheless, for complex multi-tasks or MC-PEG, direct application of MAPPO often fails to achieve satisfactory convergence\cite{yan2022relative}, necessitating the employment of advanced methodologies (e.g., hierarchical models\cite{jin2021hierarchical} or alternative training\cite{lu2023self}).

Within the framework of MARL, the Centralized Training with Decentralized Execution (CTDE) architecture acts as a foundational solution \cite{lowe2017multi}. Accordingly, many variants of CTDE \cite{zhu2022survey} have been proposed to improve the execution performance by devising a communication module to allow agents to explicitly or implicitly exchange their local information during the training.
Contingent on a communication network or leader-follower assumption, these MARL approaches generally face some communication-performance dilemmas. For example, \cite{guan2022efficient} unveils that a globally shared observation might generate a significant amount of redundant information, and even yield less competitive results than the case with partial local observation only. Therefore, it becomes critical to design some consensus inference algorithms to effectively guide the interaction between agents. In that regard, conventional algorithms\cite{amirkhani2022consensus} generally formulate an optimization problem and utilize control theory to produce a solution. Nevertheless, these algorithms lack the essential flexibility, and cannot be easily merged into MARL methods\cite{yang2021data}. Meanwhile, attributed to a black-box deep neural network (DNN), the communication module in \cite{sukhbaatar2016learning, das2019tarmac} only transmits the local information in a blunt manner, which could not unleash its potential to the full extent (e.g., unable to infer and forward global information during the execution) and fails to filter the meaningful communication content, thus being less competent to handle complex scenarios. As a remedy, the opponent modeling approaches \cite{kim2020communication, wang2021tomc} interpret the communication content as the speculated future actions of other agents. However, in partially observable scenarios, this approach suffers from speculation inconsistency, as individual agents possibly make different speculations according to their local observations\cite{xu2023consensus}. Thus, it remains challenging to infer the consensus, and perceive consistent information from diversified, limited local information.  
Alternatively, Hierarchical Reinforcement Learning (HRL) is considered to effectively deal with these underlying difficulties \cite{ pateria2021hierarchical}.
In HRL, high-level policies are cooperatively learned to focus on subgoals (e.g., to which position), while the low-level policies are designed for completing subgoal-related basic operations (e.g., specific movements, obstacle avoidance)\cite{jin2021hierarchical}.
The coordinated interplay between two-level policies often achieves effects in complex tasks that surpass the capabilities of a single-level structure\cite{peng2022ase}.
Nevertheless, integrating consensus mechanisms into HRL adds to the training difficulty, especially when agents must make decisions based on incomplete and diverse local observations.

In this paper, towards the CEFC control in a partially observed environment, we propose a novel decentralized Consensus Inference-based Hierarchical Reinforcement Learning (CI-HRL) framework.
To be specific, to tackle the global collaboration problem of agents with local observations, we incorporate a hierarchical CTDE architecture with both high- and low-levels. In particular, the high-level policy is designed to select appropriate anchor points (i.e., target positions), according to the state of neighbors and predators on top of Consensus-oriented Multi-Agent Communication (ConsMAC), while the low-level policy, which is implemented by alternative training and policy distillation, is responsible for adaptive formation navigation and obstacle avoidance mandated by the high-level decision.
Compared with the existing work, the contribution of our paper can be summarized as follows.
\begin{itemize}
\item We present a hierarchical framework for a single-pursuer-multiple-evader MC-PEG. 
Specifically, the high-level policy provides appropriate anchor points and determines the target selection policy, while the low-level policy resolves motion control (i.e., formation, navigation, and obstacle avoidance), enabling UAVs to navigate and adapt to dynamic and uncertain CEFC environments effectively.
\item On top of Alternative Training-based MAPPO (AT-M), we implement an efficient multi-policy-distilled model for low-level decentralized adaptive formation with obstacle avoidance, which is capable of adapting to agent quantity changes, reducing the training cost, and improving the obstacle avoidance and formation performance. 
\item The high-level policy learns a distributed target selection and division policy, based on an inter-agent unified understanding of the current global state provided by ConsMAC. Notably, the high-level RL-based module and ConsMAC are trained alternately to align consensus inference and policy making.
\item Through extensive simulations in both multi-agent particle environment (MPE) and software-in-the-loop (SITL) environment in Gazebo, we demonstrate the effectiveness and superiority of our framework over existing models.
\end{itemize}

To improve readability, we summarize a list of abbreviations in Table \ref{tab:keyabbreviation}.
The remainder of the paper is organized as follows. Sec. \ref{sec:related} briefly introduces the related works. Sec. \ref{sec:system} presents the system model and formulates the problem.
Sec. \ref{sec:ci-hrl} provides the details of the proposed framework. 
In Sec. \ref{sec:exp}, we introduce the experimental results and discussions. Finally, Sec. \ref{sec:final} concludes the paper.

\section{Related work}\label{sec:related}
\subsection{DRL for PEG-based UAVs System}\label{sec:PEG}

The PEG has been extensively studied in UAV systems due to its flexible requirements, and DRL has been proven effective for environmental awareness and decision control capacity \cite{xu2022autonomous,Zhang2023RealCity,Zhang2023Games}.
However, the literature above only focuses on a single, simplified evader, and is contingent on an over-idealistic full connectivity assumption. 
For example, Ref. \cite{Young2020Consensus} proposes a hierarchical system integrating flocking control and RL for multi-agents to evade a pre-defined pursuer. But it implements one invariant formation and oversimplifies the pursuer policy. Ref.  \cite{Yang2023LargScale} combines Mix-Attention and Independent PPO (IPPO) algorithm to enhance the agents' adaptation in the multi-pursuer multi-evader PEG. Notably, these methods 
neglect the cooperation among evaders and have not considered the downstream tasks such as target covering and formation maintaining. While Ref. \cite{deng2020multi} designs a collaborative pursuit-defense strategy in a fire-fighting task, it does not take account of the possible existence of obstacles and assumes full connectivity.
Different from these existing works, 
we address downstream tasks for the communication-limited UAV swarm, and aim to develop a decentralized policy that optimizes the completion of the CEFC task while considering communication constraints and collision avoidance.

\subsection{Hierarchical Reinforcement Learning}
HRL\cite{pateria2021hierarchical} excels in 
simplifying complex tasks into manageable subtasks for long-term, multi-step problem-solving.
Classical research in HRL aims to optimize the discovery and use of subtasks and to develop algorithms supporting hierarchical structure learning\cite{dayan1992feudal, dietterich2000hierarchical}. 
Option-based methods \cite{bacon2017option} enable agents to dynamically discover subtasks through environmental exploration, merging into a hierarchical policy. However, these approaches may incur higher exploration and training expenses, and the discovered subtasks might be incomplete or suboptimal, potentially impacting overall performance quality. 
In our scenario, UAVs must consider both formation obstacle avoidance and cooperative pursuit avoidance, without requiring complex skill learning. To address this, we opt for classical hierarchical policy learning, where subtasks are predefined. Here, the low-level policy is trained independently and then integrated into the training of the high-level policy. This approach avoids the mutual influence and excessive difficulty of concurrently training multiple policies 
\cite{levy2017learning}.

\subsection{Multi-Agent Reinforcement Learning with Communication}
To guarantee agents with only local information to cooperate meaningfully, CTDE is widely adopted in recent MARL methods\cite{lowe2017multi}.
However, the partial observability in a CTDE-based multi-agent environment can undermine the coordination between agents \cite{xu2023consensus}. With few exceptions like HASAC \cite{liu2024maximum}, most recent works introduce a communication module to effectively mitigate this challenge \cite{zhu2022survey}.
For example, TarMAC\cite{das2019tarmac} leverages the attention mechanism in the communication policy to aggregate the messages from their neighbors.
However, these models do not clearly define the content and significance of communication. Instead, they treat the communication network as a black box, thus reducing message interpretability and compromising their effectiveness in managing complex scenarios. To address this, recent studies have incorporated opponent modeling to better interpret communication content.
In IS\cite{kim2020communication} and ToM2C\cite{wang2021tomc}, each agent is designed to predict the future actions of its teammates, and utilizes these predictions as the substance of communication messages. However, in partially observable scenarios, agents possibly derive varied conjectures based on their individual perceptions, and this speculation heterogeneity could potentially disrupt the agent's decision-making process.
Meanwhile, MASIA\cite{guan2022efficient} and NVIF\cite{chai2023nvif} employ supervised learning and autoencoder to reduce the redundancy of communication.
Particularly, NVIF compresses local information without linking it to global information, whereas MASIA preserves the global state but assumes full communication. Under limited communication, merely transmitting raw observations through MASIA fails to convey sufficient information for overall movement trends
. In contrast, ConsMAC helps agents combine local observations and communication messages to infer a consensus on the global state, effectively addressing these challenges.

\begin{figure}[tbp]
\centering
\includegraphics[width = 0.47\textwidth]{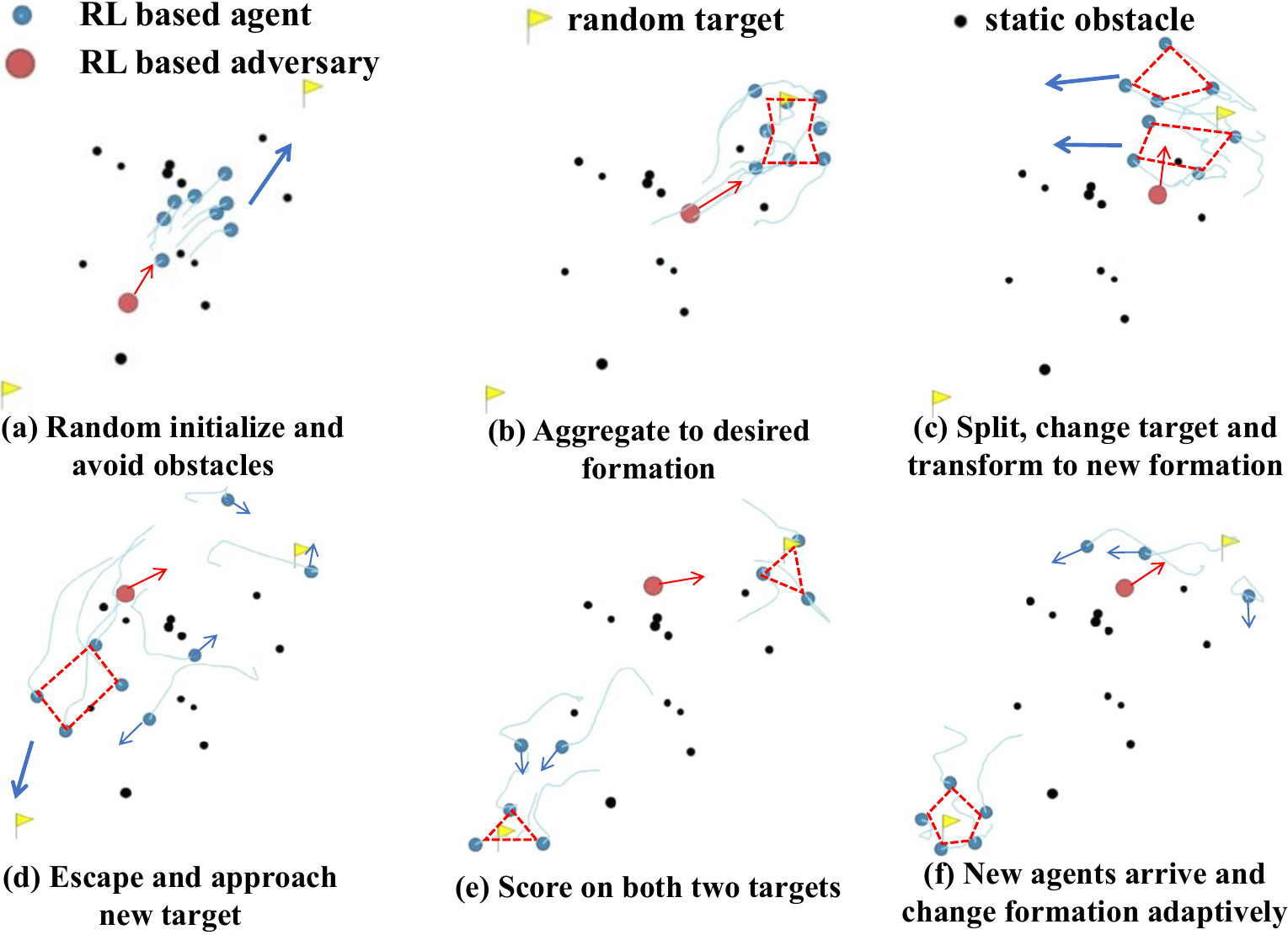}
\vspace{-.5em}
\caption{An illustration of the CEFC task, wherein the UAV swarm (in blue dots) with some formation pattern is required to automatically respond to the adversary (e.g., wildfire, in the red circle) and obstacles (in the black circle), and move towards the multiple target areas (shown as the yellow flag) to carry out their missions (e.g., dropping supplies) in a communication-limited decentralized manner.}
\label{fig:MPEOverview}
\vspace{-.5em}
\end{figure} 

\section{Preliminaries and System Model}\label{sec:system}
\subsection{System Model}
In this article, we consider a PEG scenario where a well-formed swarm of UAVs flies and attempts to reach target areas across an obstacle-cluttered space. In particular, the UAV flock is fully decentralized with a limited communication range. 

\subsubsection{Quadcopter Model}\label{sec:qmodel}
The acceleration vector of each quadrotor UAV is defined as $[u_x, u_y, u_z]^{\top}$, where $u_x$, $u_y$, $u_z$ denote the acceleration in the North-East-Down inertial frame. 
Each UAV is controlled by four control inputs $\mathbf{U}=[U_1,U_2, U_3, U_4]$ computed by the autopilot, e.g., PX4 in our work, where $U_1$ is the thrust force along the vertical direction and $U_2$, $U_3$, $U_4$ are rolling, pitching, and yawing moments respectively.
Aligned with MARL-related mainstream UAVs studies for MC-PEG\cite{hu2024transfer}, \cite{Zhang2023Games}, we assume the UAV swarm flies at a fixed altitude, by constraining ${u_z} \equiv 0$. After obtaining the desired acceleration $\textbf{u}=[{u_x}, {u_y}]$ in the horizontal direction by the proposed MARL method, it is sent to the autopilot, e.g., PX4, to calculate the thrust force $\mathbf{U}$ based on the Proportion Integration Differentiation (PID) algorithm\cite{meier2015px4}, with which the physical simulator then iterates UAV's pose based on 
the Newton-Euler formalism \cite{bresciani2008modelling}.


\subsubsection{CEFC Task Model}\label{sec:task_model}

We primarily consider a CEFC task as illustrated in 
Fig. \ref{fig:MPEOverview}, wherein a set $\mathcal{N}$ of UAVs, with $\vert \mathcal{N}\vert =N$, are required to carry out their missions at target areas (i.e., formation coverage) in a communication-limited decentralized manner, as shown in 
Fig. \ref{fig:MPEOverview}(a) and \ref{fig:MPEOverview}(b). Notably, we denote $\mathcal{T}$ as the set of target areas, and each target area $\mathfrak{T}\in \mathcal{T}$ can be represented as $\mathfrak{T}=(\textbf{p}_\mathfrak{T}, \kappa_\mathfrak{T}^{(t)}) $, where $\textbf{p}_\mathfrak{T}$ denotes the position and $\kappa_\mathfrak{T}^{(t)}$ denotes the urgency of the target area, which decreases as the agents arrive.
Besides, the adversary $\mathfrak{A}$ should prevent the agents from reaching the target areas, and in this work, due to the algorithmic maturity, domain suitability for continuous control, and implementation practicality with reduced hyperparameter sensitivity \cite{Yan2023Target}, a default PPO policy is primarily used for the adversary to trace the nearest group.

In case of the possible adversary (Fig. \ref{fig:MPEOverview}(c)) and obstacles (Fig. \ref{fig:MPEOverview}(d) and \ref{fig:MPEOverview}(e)), the set of UAVs switch among a set of pre-defined formation patterns $\{\Delta_c | \forall c \in \mathcal{C}\}$, where $c$ represents the agent quantity in formation $\Delta_c$ and $\mathcal{C}$ denotes the set of possible patterns.
Hence, at each time step $t$, each agent $i$ needs to recognize its $c-1$ teammates cooperatively and spontaneously determine one formation pattern $c$, resulting in $\chi^{(t)}$ groups with each group $\mathcal{N}_k, k\in \{1,\cdots, \chi^{(t)}\}$ of agents satisfying $\vert \mathcal{N}_k\vert = n_k$, $n_0$ isolated agents (i.e., no sufficient number of neighbors within the observation range to form any pattern in $\mathcal{C}$) and ${n}_{0}+{n}_{1}+\cdots+{n}_{\chi^{(t)}}=N$. 
In response to changes in the number of neighbor agents (i.e., from $3$ in Fig. \ref{fig:MPEOverview}(e) to $5$ in Fig. \ref{fig:MPEOverview}(f)), the UAV shall switch the formation pattern automatically.

To accomplish the CEFC task, we formulate the problem as a Decentralized Partially Observable Markov Decision Process (Dec-POMDP), which is defined as $\langle\mathcal{N}, \mathcal{S}, \mathcal{U}, P, \mathcal{Z}, \textbf{R}, \Omega, \gamma \rangle$.  
In the CEFC task, $\mathcal{S}$ denotes the global state space while $\mathcal{U}$ is the homogeneous action space for a single agent. 
At each time step $t$, owing to the scant ability of perception against the colossal environment, each agent $i\in \mathcal{N}$ obtains a local state $\textbf{z}_{i}^{(t)} \in \mathcal{Z}$ via the local state function $\Omega\left(\textbf{z}_{i}^{(t)} \mid \textbf{s}^{(t)}, i\right): \mathcal{S} \times \mathcal{N} \times \mathcal{Z} \rightarrow[0,1]$ instead of the state $\textbf{s}^{(t)} \in \mathcal{S}$ at each time-step, and adopts an action $\textbf{u}^{(t)}_{i} \in \mathcal{U}$ according to the individual policy $\bm{\pi}_{i}\left(\cdot \mid \textbf{z}_{i}^{(t)} \right): \mathcal{Z} \times \mathcal{U} \rightarrow[0,1]$.
The joint action $\textbf{u}^{(t)} =[\textbf{u}_{1}^{(t)} , \textbf{u}_{2}^{(t)} , \cdots, \textbf{u}_{N}^{(t)} ]$ taken at the current state $\textbf{s}^{(t)} $ makes the environment transit into the next state $\textbf{s}^{(t+1)}$ according to the function $P\left(\textbf{s}^{(t+1)} \mid \textbf{s}^{(t)}, \textbf{u}^{(t)}\right): \mathcal{S} \times \mathcal{U} \times \mathcal{S} \rightarrow[0,1]$. All agents share a global reward function $\textbf{R}(\textbf{s}^{(t)}, \textbf{u}^{(t)}): \mathcal{S} \times \mathcal{U} \rightarrow \mathbb{R}$ with a discount factor $\gamma$ and agents need to maximize the discounted accumulated reward $\mathbb{E} [{\textstyle \sum_{t}}\gamma ^t \textbf{R}^{(t)}]$. 
Consistent with the Dec-POMDP framework, we specify the elements as follows.
\begin{itemize}
\item \textit{Local State}:
The local state $\textbf{z}_{i}^{(t)}$ should encompass the information of neighbors, target areas, and the adversary.
For brevity, we denote the relative position and velocity of agent $j$ with respect to $i$ as $\textbf{p}_{i\to j}$ and $\textbf{v}_{i\to j}$.
The observation of target areas and the adversary can be denoted as $\textbf{o}_{\mathrm{tar} ,i}^{(t)}=\{\textbf{p}_{i\to \mathfrak{T}}^{(t)}, \kappa^{(t)}_\mathfrak{T}|\mathfrak{T} \in \mathcal{T} \}$ and $\textbf{o}_{\mathrm{adv},i}^{(t)}=[\textbf{p}_{i\to \mathfrak{A} }^{(t)},\textbf{v}_{i\to \mathfrak{A} }^{(t)}]$, respectively. 
Meanwhile, agent $i$ obtains some direct observation $\textbf{o}_{\mathrm{nei},i}^{(t)}=\{\textbf{p}_{i\to j}^{(t)},\textbf{v}_{i\to j}^{(t)}| \forall j\in \bm{\xi} _i^{(t)}\}$ of their neighbors through communications and receives exchanged messages $\textbf{M}_{\mathrm{nei},i}^{(t)}=\{ \textbf{m}_j^{(t)}| \forall  j\in \bm{\xi} _i^{(t)}\}$\footnote{The detailed procedure to acquire these messages shall be discussed in Sec. \ref{sec:high}.}, where $\textbf{m}_j^{(t)}$ is a learnable vector to be communicated and 
$\bm{\xi} _i^{(t)}$ denotes the set of agents satisfying that the Euclidean distance $\|\textbf{p}_{i\to j}^{(t)}\|$ is less than a maximum observation distance $\delta_{\obs}$ (i.e., $\|\textbf{p}_{i\to j}^{(t)}\| < \delta_{\obs}$).
For the sake of simplicity, the local observation of agent $i$ is defined as $\textbf{o}_i^{(t)} = [\textbf{o}_{\mathrm{tar},i}^{(t)},\textbf{o}_{\mathrm{adv},i}^{(t)},\textbf{o}_{\mathrm{nei},i}^{(t)}]$.
Besides, the detection results $\textbf{d}_{i}^{(t)}=[d_{i1}^{(t)},...,d_{iM}^{(t)}]$ of $M$ LiDARs\cite{yan2023collision} with angle resolution $2\pi/M$ are also used as part of the local state to help the agent avoid obstacles. 
Thus, the local state of agent $i$ is summarized as $\textbf{z}_i^{(t)}=[\textbf{o}_i^{(t)}, \textbf{M}_{\mathrm{nei},i}^{(t)},\textbf{d}_{i}^{(t)}]$. 

\item \textit{Action}:
As mentioned in Sec. \ref{sec:qmodel}, based on the local state $\textbf{z}_i^{(t)}$, each agent sets its acceleration $\textbf{u}_i^{(t)}=[{u}_{x_i}^{(t)},{u}_{y_i}^{(t)}] \in \mathcal{U}$ following policy 
$\bm{\pi}_{i}\left(\cdot \mid \textbf{z}_{i}^{(t)}\right)$ individually to complete CEFC task.

\item \textit{Reward Function}: For the CEFC task, the reward function is designed to summarize multiple ingredients with weights $\omega_{f}$, $ \omega_{n}$, $ \omega_{t}$, $\omega_{e}$, $\omega_{c}$ as
\begin{equation}
\label{eq:totalreward}
    \textbf{R}^{(t)} = \omega_f R_\mathrm{f}^{(t)} + \omega_n R_\mathrm{n}^{(t)} + \omega_t R_\mathrm{t}^{(t)} + \omega_e R_\mathrm{e}^{(t)} + \omega_c R_\mathrm{c}^{(t)},
\end{equation}
where the Hausdorff Distance \cite{pan2022flexible} (HD)-based \textit{formation reward} $R_\mathrm{f}^{(t)}$ measures the topological distance between the current and the expected formation for each group, while the \textit{navigation reward} $R_\mathrm{n}^{(t)}$, weighted by the urgency of target areas, incentivizes agents to efficiently reach target areas. The \textit{task accomplishment reward} $R_\mathrm{t}^{(t)}$ is used to quantify the progress of the formation coverage.
The \textit{evasion reward} $R_\mathrm{e}^{(t)}$ and \textit{collision avoidance reward} $R_\mathrm{c}^{(t)}$, which serve as a critical safety mechanism, penalize agents in close proximity to the adversary and obstacles.
We leave their specific expressions in Appendix \ref{sec:specificreward}. 

\end{itemize}

\begin{figure*}[tbp]
\vspace{-2.5em}
    \centering
\includegraphics[width = 0.73\textwidth]{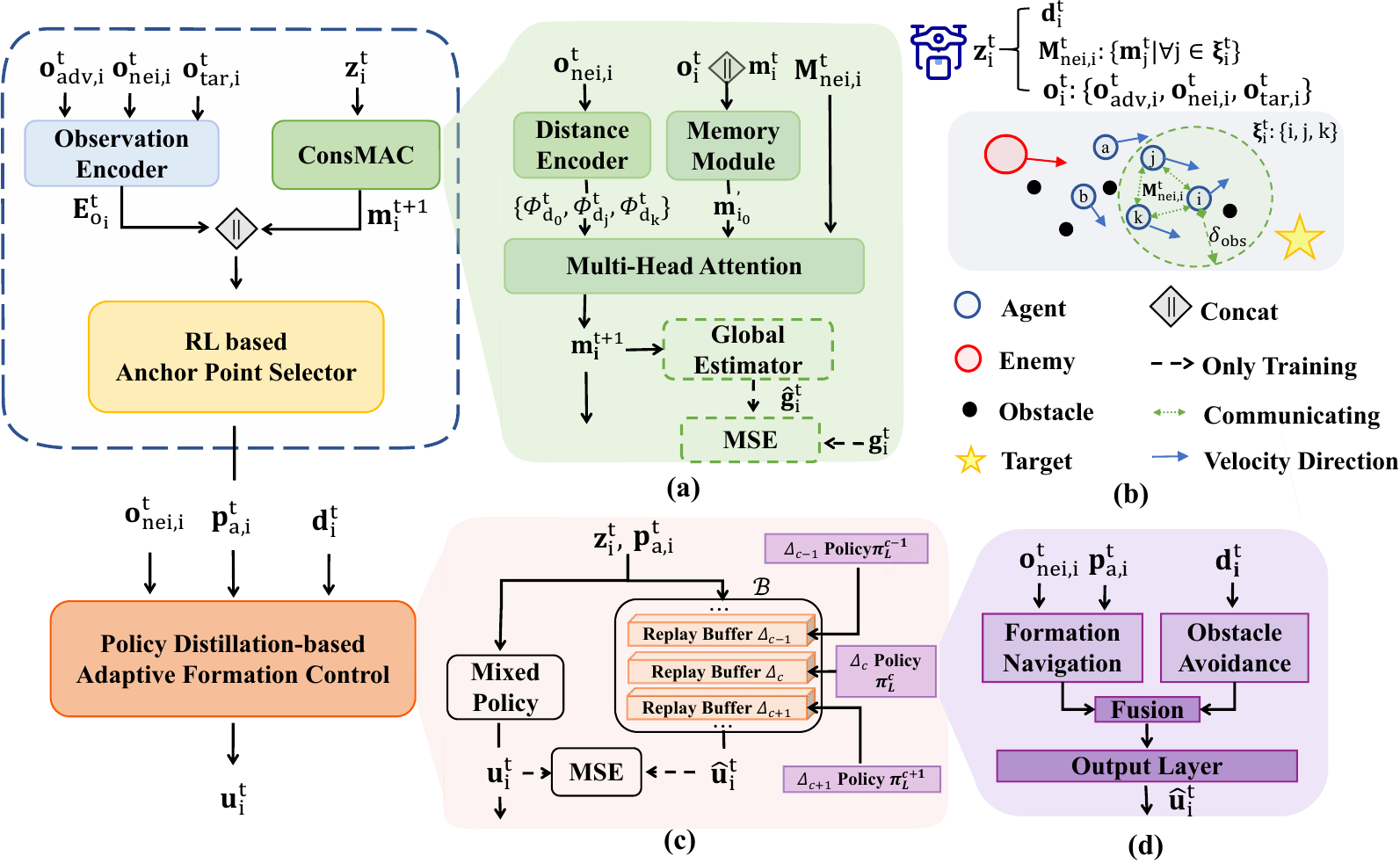}
\vspace{-1.1em}
\caption{The overview of CI-HRL.
\textbf{High-level Policy, in Section \ref{sec:high}}, generates anchor point $\textbf{p}_{a,i}^{(t)}$:
(a) The ConsMAC Module: Process local state $\textbf{z}_i^{(t)}$ to infer global state ${\textbf{g}}_i^{(t)}$ and generate the consensus message $\textbf{m}_i^{(t+1)}$,
(b) Communication process for $\textbf{M}_{\mathrm{nei},i}^{(t)}$ under limited range;
\textbf{Low-level Policy, in Section \ref{sec:low}}, outputs specific actions $\textbf{u}_i^{(t)}$ for task execution: (c) Policy distillation process for adaptive formation control,  (d) AT-M for single-formation policy.
}
\label{fig:overview}
\vspace{-1.7em}
\end{figure*}  
\subsection{Multi-Agent Proximal Policy Optimization}\label{sec:mappo}
To guide all agents to maximize the discounted accumulated reward $\mathbb{E} [{\textstyle \sum_{t}}\gamma ^t \textbf{R}^{(t)}]$,  MAPPO\cite{yu2022surprising}, which combines the single-agent PPO\cite{schulman2017proximal} and CTDE to learn the policy  $\bm{\pi}_{\theta_i}(\cdot |\textbf{z}_i^{(t)})$ ($\forall i$) and value function $V_\phi(\textbf{s}^{(t)}):\mathcal{S} \rightarrow \mathbb{R} $ parameterized by $\theta_i$ and $\phi$, is used. Consistent with PPO, MAPPO maintains the old-version $\theta_{i,\old}$ and $\phi _{\old}$, and uses $\theta_{i,\old}$ to interact with the environment and accumulate samples. Furthermore, $\theta_i $ and $\phi $ are periodically updated to maximize
\begin{align}    
    J_{\pi_i}^{(t)}(\theta_i)  = & \min \left( \beta_i^{(t)} \hat{A}^{(t)},\text{clip}\left(\beta_i^{(t)}, 1-\varepsilon, 1+\varepsilon\right) \hat{A}^{(t)}\right),\nonumber\\
    J_V^{(t)}(\phi) = &-\left (V_\phi (\textbf{s}^{(t)})- (\hat{A}^{(t)}+ V_{\phi _{\old}}(\textbf{s}^{(t)}) \right )^2,\label{eq:ppo}
\end{align}
where $\beta_i^{(t)} = \nicefrac{\bm{\pi}_{\theta_i}\left(\textbf{u}_i^{(t)}|\textbf{z}_i^{(t)}\right)}{\bm{\pi}_{\theta_{i,\old}}\left(\textbf{u}_i ^{(t)}|\textbf{z}_i^{(t)}\right)}$, the clipping function $\text{clip}\left(\beta_i^{(t)}, 1-\varepsilon, 1+\varepsilon\right)$ constrains $\beta_i^{(t)} \in [1-\varepsilon, 1+\varepsilon]$, and $\hat{A}^{(t)} = \sum _{l=0}^{T-t-1} (\gamma \lambda)^l \delta^{(t+l)}$ is the generalized advantage estimation (GAE)\cite{schulman2015high} 
with $\delta^{(t)}  =  \textbf{R}^{(t)} + \gamma V_{\phi_{\old}}(\textbf{s}^{(t+1)})-V_{\phi_{\old}}(\textbf{s}^{(t)})$. Notably, $\varepsilon$ and $\lambda$ are the hyperparameters.
Thus, the final optimization objective of MAPPO can be given by
\begin{equation}
\label{eq:mappo}
    J_{\mathrm{MAPPO}}=\mathbb{E}_{i,t} \left[ J_{\pi_i}^{(t)}(\theta_i )+J_V^{(t)}(\phi)+\alpha \mathcal{H}(\pi_{\theta_i} (\cdot |\textbf{z}_i^{(t)})) \right],
\end{equation}
where $\alpha$ is coefficient, and $\mathcal{H}$ is the entropy function. 

\subsection{Problem Formulation}\label{sec:pf}
We expect the MARL-driven agents to minimize the urgency of the target areas while avoiding collisions with obstacles and the PPO-driven adversary. 
Meanwhile, the parameters of each agent are shared during training to improve the learning efficiency, which implies $\bm{\pi}_{\theta_i}=\bm{\pi}_{\theta}$ ($\forall i\in \mathcal{N}$).
Accordingly, we propose a system utility function $\mathfrak{J}$ as the objective of policy optimization, which can be expressed as
\begin{equation}
\label{mappo}
    \max_{\bm{\pi}_{\theta}}\mathfrak{J} = \max_{\bm{\pi}_{\theta}} \mathbb{E}_t \left[\textbf{R}^{(t)}\big|\bm{\pi}_{\theta}\big(\cdot |\textbf{z}^{(t)}\big)\right].
\end{equation}
Basically, we follow the CTDE architecture and leverage MAPPO for multi-agent control. 
Recalling that $\textbf{z}_i^{(t)}=[\textbf{o}_i^{(t)}, \textbf{M}_{\mathrm{nei},i}^{(t)},\textbf{d}_{i}^{(t)}]$ ($\forall i$),
the communicated information (i.e., $\textbf{m}^{(t)}_i$) shall impact the final performance. Meanwhile, classical solutions to CEFC rely on a centralized leader and cannot be directly applied to distributed agents. Therefore, to train a policy that can be executed in a completely distributed manner, we resort to a consensus inference module ConsMAC to compute $\textbf{m}^{(t)}_i$ in Sec. \ref{sec:high} and propose a distillation-based, agent number-adaptable distributed solution.

\section{Consensus Inference-based Hierarchical Reinforcement Learning}
\label{sec:ci-hrl}
\subsection{Overview of the CI-HRL framework}\label{sec:overview}
In this section, we present the overview of the proposed CI-HRL framework, which capably addresses the challenges of cooperative evasion from the adversary and the execution of formation coverage tasks under the constraints of limited communication. 
As shown in Fig. 2, CI-HRL implements a decentralized, real-time inference of the environment, and
enables the UAV swarm to dynamically split and reorganize to accomplish the CEFC task. Specifically, for each agent $i$, CI-HRL automatically determines the anchor point $\textbf{p}_{a,i}\in \mathcal{P}_a$ as the temporary target location over a specified duration, where $\mathcal{P}_a$ denotes a candidate set of discrete or continuous coordinates. Hence, to strike a balance between evasion and mission execution, a subset of UAVs may select anchor points proximal to target areas for formation coverage, while others choose anchor points at a distance from predators to evade. 

The whole training process consists of two stages, and the final policy can be represented as a combination of two-level policies, $\bm{\pi}_{\Theta} = [\bm{\pi}_L, \bm{\pi}_H]$, where $\Theta$ denotes the collective parameters of multiple DNNs. 
During the low-level stage, each agent $i$ should learn a policy $\bm{\pi}_L\left(\textbf{u}_i|\textbf{p}_{a,i},\textbf{z}_i\right)$, which is a distillation-based adaptive formation solution AT-M, to output continuous acceleration $\textbf{u}_i$ to control its journey to a specific anchor point $\textbf{p}_{a,i}$ while maintaining formation with its neighbors.
Then, at the high-level stage, a ConsMAC-based, decentralized high-level policy $\bm{\pi}_H(\textbf{p}_{a,i}|\textbf{z}_i)$ is trained to output the anchor point for guiding the low-level policy, thus jointly accomplishing the CEFC task.
It should be noticed that the low-level policy solely accounts for formation with neighbors and navigation toward the anchor point, implying that when some UAVs select congruent anchor points that significantly differ from the rest of the swarm, these UAVs will automatically secede from the main group to form a new sub-group. Consequently, the formation pattern, to which each UAV belongs, is indirectly determined by the high-level decision-making.
To simplify the following presentation, we denote Multi-Layer Perceptron (MLP)-based DNN as $\mathcal{F}(\cdot)$.

\subsection{Low-level Policy}\label{sec:low}
In this section, we design an effective low-level policy that allows agents to constitute some pre-defined formations and move towards the anchor point in a communication-limited, decentralized manner.
Notably, we start with a low-level policy for a specific formation while ignoring the target areas as well as the adversary in Sec. \ref{sec:single}. During the training, all agents belong to the same group (i.e., $\chi^{(t)}$ is set to $1$). Afterward, more general cases for flexible formation patterns will be investigated in Sec. \ref{sec:adaptive}.

\subsubsection{Low-Level Objective}
As mentioned earlier, the low-level policy considers how to avoid obstacles and travel to the anchor point with formation, yielding the acceleration control of the agent. Therefore, the input only includes the observation of neighbors $\textbf{o}_{\mathrm{nei},i}^{(t)}$, LiDAR detection results $\textbf{d}_{i}^{(t)}$, and the anchor point $\textbf{p}_a$, which during the training are the same and randomly given by the environment in each episode. In Sec. \ref{sec:high}, we will explain how to compute appropriate anchor points for practical execution. 
Meanwhile, since the low-level policy focuses more on the navigation and formation, the task and evasion reward (i.e., $R_\mathrm{t}$ and $R_\mathrm{e}$) in \eqref{eq:totalreward} are ignored in this sub-section. Therefore, the low-level optimization objective can be simplified as 
    $\max_{\bm{\pi}_L}\mathfrak{J}_L = \max_{\bm{\pi}_L} \mathbb{E}_t [\textbf{R}_\mathrm{L}^{(t)}\big|\bm{\pi}_L]$,
where $\textbf{R}_\mathrm{L}^{(t)} = \omega_{f} R_\mathrm{f}^{(t)}+ \omega_n R_\mathrm{n}^{(t)} + \omega_c R_\mathrm{c}^{(t)}$. Moreover, given that we focus on anchor points instead of target areas, the navigation reward given by \eqref{eq:nareward} in Appendix \ref{sec:nav_reward} can be regarded as $R_\mathrm{n}^{(t)} = -\|\bar{\textbf{p}}^{(t)}-\textbf{p}_a\|$, where $\bar{\textbf{p}}^{(t)}$ is the center of agents.

\subsubsection{AT-M for Single-Formation Policy}\label{sec:single}
Despite the remarkable performance in small-scale, fully connected multi-agent formation tasks, MAPPO 
\cite{yu2022surprising} in Section \ref{sec:mappo} faces significant difficulties in training formation and navigation policy directly in a partially observable environment with random obstacles. For example, our results show that by assigning a small weight $\omega_c$ to collision avoidance reward, MAPPO converges to a well-formed policy with frequent collisions. On the contrary, the policy ends up with poor formation competence for enlarging $\omega_c$. Motivated by these facts, we adopt alternative training-based MAPPO (AT-M) to fuse multiple coupled constraints and better balance the tradeoff therein.

Beforehand, consistent with the conventional MAPPO framework, the corresponding value function $V_{\phi_L}$ is implemented by using an MLP \cite{yu2022surprising}. 
As shown in Fig. 2(d), for a fixed formation $\{\Delta_c | \forall c \in \mathcal{C}\}$, AT-M aims to obtain a policy network $\bm{\pi}_L^c$ parameterized by $\Theta_L^c = [\theta_{Lf}^c, \theta_{La}^c, \theta_{Lo}^c]$, which consists of three MLP-based parts, namely the formation \& navigation module $\mathcal{F}_{\theta_{Lf}^c}$, the obstacle avoidance module $\mathcal{F}_{\theta_{La}^c}$ and the output layer $\mathcal{F}_{\theta_{Lo}^c}$.
In particular, for each time step $t$, $\mathcal{F}_{\theta_{Lo}^c}$ is used to sample an action $\textbf{u}_i^{(t)}\sim \operatorname{Normal}(\boldsymbol{\mu}_i^{(t)},\boldsymbol{\sigma}_i^{(t)})$ as the final output by computing the mean and variance of the Gaussian distribution as
\begin{equation}
\label{eq:lowlevel}
        \boldsymbol{\mu}_i^{(t)}, \boldsymbol{\sigma}_i^{(t)} = \mathcal{F}_{\theta_{Lo}^c}(\mathcal{F}_{\theta_{Lf}^c}(\textbf{p}_a, \textbf{o}_{\mathrm{nei},i}^{(t)})+\mathcal{F}_{\theta_{La}^c}(\textbf{d}_{i}^{(t)})).
\end{equation}
AT-M first independently trains DNNs parameterized by $\Theta_{L}^c$ in an environment without and with the involvement of obstacles (i.e., some $\omega_c\neq 0$), and obtains two sets of corresponding parameters $\Theta_{L,1}^c$ and $\Theta_{L,2}^c$ respectively. 
On this basis, AT-M fine-tunes a merged policy parameterized by $[\theta_{Lf,1}^c, \theta_{La,2}^c, \theta_{Lo,1}^c]$ in the original environment to achieve the policy $\bm{\pi}_L^c$.
Correspondingly, the swarm is endowed with the capacity to reach the anchor point in a fixed formation $\Delta_c$ while avoiding obstacles. 
Due to the independent pre-training of modules across disparate environments, we regard this approach as alternative training.

    

\begin{figure}[!tbp]
\centering
\includegraphics[width=0.42\textwidth]{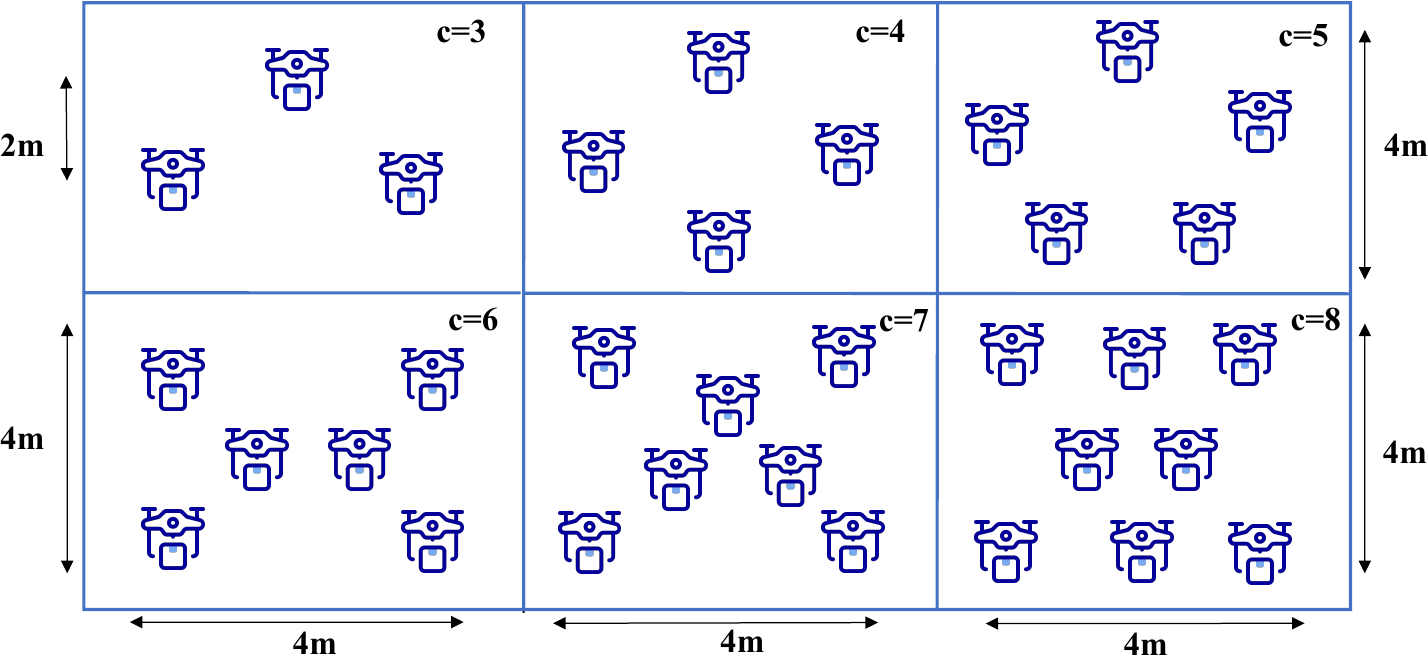}
\vspace{-1em}
\caption{An illustration of the pre-defined formations corresponding to different numbers of UAVs (i.e., $c\in \mathcal{C} = \{3, 4, 5, 6, 7, 8\}$).}
\label{fig:forma}
\vspace{-1em}
\end{figure}
\subsubsection{Policy Distillation-based Adaptive Formation Control}\label{sec:adaptive}
Following Sec. \ref{sec:single}, we can obtain a set of policies  $\{\bm{\pi}_L^c| \forall c \in \mathcal{C}\}$ respectively for formation $\Delta_{c}$ depicted in Fig. \ref{fig:forma}.
In this part, we regard these policies as teacher models (i.e., a teacher model $\bm{\pi}_L^c$ instructs the formation $\Delta_c$), and utilize policy distillation\cite{rusu2015policy} to obtain a mixed-formation policy, so as to reduce local memory occupation.
As shown in Fig. 2(c), we collect both inputs and outputs  of policies $\{\bm{\pi}_L^c|\forall c \in \mathcal{C} \}$ by constituting a replay memory $\mathcal{B}= \{\langle\textbf{p}_a^c, \textbf{z}^{{c}}, \textbf{u}^{{c}}\rangle_{\times \Lambda}|\forall c \in \mathcal{C} \}$, 
where $\textbf{u}^{c}$ is generated by a learned teacher model $\bm{\pi}_L^{c}$ to form $\Delta_c$ and $\Lambda$ is the capacity of replay buffer. 
Considering the dimension of observation $\textbf{z}^{c}$ is determined by the agent number $c$, we align the vectors of observations in different formations by a zero-padding operation.
In each training episode, memories (i.e., $\textbf{z}$) from different teacher models are fed into the student model $\bm{\pi}_L$ simultaneously to calculate the corresponding $\hat{\textbf{u}} = \bm{\pi}_L(\textbf{p}_{a},\textbf{z})$. 
Afterwards, we train $\bm{\pi}_L$ by minimizing the Mean-Squared-Error (MSE) loss
\begin{equation}
\label{eq:pdloss}
    \mathcal{L}_{\mathrm{PD}}\left(\Theta_L\right) = \sum\nolimits_{\textbf{u}\in \mathcal{B}}\left \| \textbf{u} -\hat{\textbf{u}} \right \| ^2_2 ,
\end{equation}
where $\Theta_L$ is the parameter of $\bm{\pi}_L$. 
After updating $\Theta_L$ through \eqref{eq:pdloss}, the mixed-formation policy set $\{\bm{\pi}_L^c | \forall c \in \mathcal{C}\}$ is merged to one student policy $\bm{\pi}_L$, which significantly saves the usage of agents' memory. 

To sum up, we describe the training procedure of the low-level policy in Algorithm \ref{alg:low}.

\begin{algorithm}[tbp]
\caption{The Training of Low-Level Policy}\label{alg:low}
\begin{algorithmic}[1]
    \STATE Initialize the set of quantities $\mathcal{C}$.\\
    \textbf{// AT-M for each formation pattern}
    \FOR{each $c \in \mathcal{C}$}
        \STATE \textbf{// Step 1}
        \STATE Initialize the environment corresponding to $c$;
        \STATE Train the policy and value function parameterized by $\Theta_{L,1}^c = [\theta_{Lf,1}^c, \theta_{La,1}^c, \theta_{Lo,1}^c]$ and $\phi_{L,1}^c$ with random initialization by MAPPO; 
        \STATE \textbf{// Step 2}
        \STATE Initialize the environment with random obstacles;
        \STATE Train the policy and value function parameterized by  $\Theta_{L,2}^c = [\theta_{Lf,2}^c, \theta_{La,2}^c, \theta_{Lo,2}^c]$ and $\phi_{L,2}^c$ with random initialization by MAPPO; 
        \STATE \textbf{// Step 3}
        \STATE Initialize the environment with random obstacles;
        \STATE Fine-tune the the policy and value function parameterized by $\Theta_{L}^c = [\theta_{Lf,1}^c, \theta_{La,2}^c, \theta_{Lo,1}^c]$ and $\phi_{L,2}^c$ by MAPPO; 
    \ENDFOR \\
    \textbf{// Policy Distillation}
    \STATE Initialize the replay memory $\mathcal{B}\leftarrow\varnothing$ and the student policy $\bm{\pi}_L$ with random parameters $\Theta_{L}$, the training batch size $N_{b}$; 
    \STATE Collect tuples $\langle \textbf{p}_{a}, \textbf{z}, \textbf{u}\rangle$ in each environment by $\{\bm{\pi}_L^c | \forall c \in \mathcal{C}\}$ respectively and store them in $\mathcal{B}$;
    \FOR{each policy distillation epoch}
        \STATE Sample a batch of $N_{b}$ tuples from $\mathcal{B}$;
        \STATE Calculate $\hat{\textbf{u}}$ based on $\textbf{p}_{a}$ and $\textbf{z}$;
        \STATE Update $\Theta_L$ according to \eqref{eq:pdloss} via gradient descent and Adam optimizer;
    \ENDFOR
    \ENSURE The trained $\Theta_L$;
\end{algorithmic}
\end{algorithm}
\vspace{-0.1 em}

\vspace{-0.2 em}
\subsection{High-level Policy}\label{sec:high}
Contingent on the low-level policy $\bm{\pi}_L$ for yielding the formation and obstacle avoidance action $\textbf{u}$, the MAPPO-based
high-level policy $\bm{\pi}_H$ can avoid the cumbersomeness of considering underlying tasks and put more emphasis on selecting an appropriate anchor point $\textbf{p}_a$, which fully takes account of evading the adversary for accomplishing the CEFC task. Similar to the low-level policy, the corresponding value function $V_{\phi_H}$ is implemented by MLP as well. 
However, different from MAPPO, as illustrated in 
Fig. 2(a), the high-level policy introduces a novel, supervised-learning-based consensus inference method ConsMAC, which can provide global consensus and complement the RL-based target selector to calculate certain decisions from a more global perspective.

\subsubsection{High-Level Objective}
Since the low-level policy saves the high-level policy from the formation and collision avoidance, 
the optimization objective of high-level policy can be represented as $\max_{\bm{\pi}_H}\mathfrak{J}_H = \max_{\bm{\pi}_H} \mathbb{E}_t [\textbf{R}_\mathrm{H}^{(t)}|\bm{\pi}_H, \bm{\pi}_L]$, 
where $\textbf{R}_\mathrm{H}^{(t)} = \omega_{n} R_\mathrm{n}^{(t)}+ \omega_t R_\mathrm{t}^{(t)} + \omega_e R_\mathrm{e}^{(t)}$. 
Recalling $\textbf{z}_i^{(t)}=[\textbf{o}_i^{(t)}, \textbf{M}_{\mathrm{nei},i}^{(t)},\textbf{d}_{i}^{(t)}]$, the maximization of this high-level objective lies in the effectiveness of exchange messages $\textbf{M}_{\mathrm{nei},i}^{(t)}=\{ \textbf{m}_j^{(t)}| \forall  j\in \bm{\xi} _i^{(t)}\}$.

\subsubsection{ConsMAC-based Anchor Point Selection} \label{sec:hl_cons}
Targeted at effectively aggregating observations $\textbf{o}_i$ and messages $\textbf{M}_{\mathrm{nei},i}$ received from neighbors, ConsMAC is carefully calibrated and encompasses the following parts. Firstly, to incorporate historical information, each agent $i$ adopts a GRU-alike memory module $\mathcal{F}_{\psi_M}$ parameterized by $\psi_M$ to process the concatenation of its own message $\textbf{m}_{i_0}^{(t)}$ and the local observation $\textbf{o}_i^{(t)}$, that is,
\begin{equation}
\label{eq:mem}
    \textbf{m}'_{i_0}= \mathcal{F}_{\psi_M}([\textbf{m}_{i_0}^{(t)}||\textbf{o}_i^{(t)}]),
\end{equation}
where $ \|$ represents the concatenation operation and the dimensions of both $\textbf{m}_{i_0}^{(t)}$ and $\textbf{m}'_{i_0}$ are the same. Meanwhile, for simplicity of representation, for each agent $i\in \mathcal{N}$, $i_j$ denotes the $j$-th nearest neighbor while $i_0$ indicates itself. Next, in resemblance to the positional encoding in Transformer\cite{vaswani2017attention}
, we use a learnable distance encoder to distinguish messages from agents with different distances and correspondingly assign different attention weights. Mathematically,
\begin{equation}
\label{eq:DE}
    \begin{aligned}
        \mathbf{E}_{m_{i}}^{(t)} = \left[ \textbf{m}'_{i_0}||\Phi_{d_0}^{(t)}\right],
        \mathbf{E}_{m_{-i}}^{(t)} = \left[\textbf{m}_{i_1}^{(t)}||\Phi_{d_1}^{(t)}, 
             \dots, \textbf{m}_{i_k}^{(t)}||\Phi_{d_k}^{(t)}\right]^{\top},\\
    \end{aligned}
\end{equation}
where $\Phi_{d_j}^{(t)}= \sqrt{\frac{1}{D}} \ [ \cos (w _1 \|\textbf{p}_{i\to {i_j}}^{(t)}\|),\dots ,\cos (w _{D} \|\textbf{p}_{i\to {i_j}}^{(t)}\|) \ ] ^{\top}$, 
$\psi_D = [ w _1,\dots , w_D]$ are the trainable parameters, $D$ is the dimension of the latent space and $k=|\bm{\xi} ^{(t)}_i|$ denoting the number of neighbors. Then, each agent can aggregate a latent vector $\textbf{m}_i^{(t)}$ as the message for time-step $t+1$ by
\begin{equation}
\label{eq:AL}
    \begin{aligned}
        \textbf{m}_i^{(t+1)} = \operatorname{MHA}_{\psi_{A}}
        (\mathbf{E}_{m_i}^{(t)},\mathbf{E}_{m_{-i}}^{(t)},\mathbf{E}_{m_{-i}}^{(t)}),
    \end{aligned}
\end{equation}
where $\operatorname{MHA}$ is a multi-head attention layer\cite{vaswani2017attention} parameterized by $\psi_A$.
Furthermore, we infer the global information by a global estimator $\mathcal{F}_{\psi_E}$ parameterized by $\psi_E$, and the estimated state embedding can be written as
\begin{equation}
\label{eq:GE}
    \begin{aligned}
        &\hat{\textbf{g}} _i^{(t)}  = \mathcal{F}_{\psi_E}({\textbf{m}_i^{(t+1)}}).
    \end{aligned}
\end{equation}
We leverage global state $\textbf{g}^{(t)}_i \in \textbf{s}^{(t)}$ as the label for supervised learning, and adopt MSE as the loss function. Therefore, the loss function of ConsMAC can be formulated as

\begin{equation}
\label{eq:celoss}
    \begin{aligned}
        \mathcal{L}_{\mathrm{ConsMAC}}(\Psi)=  \mathbb{E}_{i,t}\left [|| \hat{\textbf{g}}_i^{(t)} - \textbf{g}_i^{(t)} ||^2\right],
    \end{aligned}
\end{equation}
where $\Psi=[\psi_D,\psi_M,\psi_A,\psi_E]$.  
Notably, the specific choice of global state $\textbf{g}^{(t)}_i$ could be rather flexible such as the anchor points of all agents $\textbf{g}^{(t)}_i = \{\textbf{p}_{a,j}^{(t)}|\forall j \in \mathcal{N}\}$ or the observations of all agents $\textbf{g}^{(t)}_i = \{\textbf{p}_{i\to j}^{(t)}, \textbf{v}_{i\to j}^{(t)}|\forall j \in \mathcal{N}\}$, denoted as ConsMAC-A and ConsMAC-O, respectively. We evaluate both methods in Sec. \ref{sec:exp} and validate the scalability of ConsMAC.
In this way, minimize \eqref{eq:celoss} ensures $\mathcal{F}_{\psi_E}({\textbf{m}_i^{(t+1)}}) \rightarrow \textbf{g}_i^{(t)}$, and the intermediate output of ConsMAC in the local perspective can implicitly embed the global information of all agents (i.e., $\textbf{m} = \mathcal{F}^{-1}(\textbf{g}, \textbf{o})$)
, rather than merely aligning local observations in methods like NVIF\cite{chai2023nvif} (i.e., $\textbf{m} \approx \mathcal{F}^{-1}(\textbf{o})$). Due to the uniqueness of the global state, such a procedure can be interpreted as the establishment of consensus among neighbors. 


Accordingly, the RL-based policy, which encompasses an observation encoder and an RL-based selector, yields a suitable anchor point corresponding to the observations and established consensus (i.e., $\textbf{o}_i^{(t)}$ and $\textbf{m}_i^{(t+1)}$). In particular, the embedding vector $\mathbf{E}_{o_i}^{(t)}$ can be obtained by the MLP-based observation encoder $\mathcal{F}_{\theta_{HO}}$ and the RL-based selector $\mathcal{F}_{\theta_{HS}}$ is used to calculate and sample the anchor point as the final output
\begin{equation}
    \label{eq:PE}
    \textbf{p}_{a,i}^{(t)} \sim \mathcal{F}_{\theta_{HS}}(\mathbf{E}_{o_i}^{(t)}, \textbf{m}_i^{(t+1)}).
\end{equation}
As mentioned in Sec. \ref{sec:overview}, we can treat $[\Theta_H,\Psi]$ as the parameters of the high-level policy $\bm{\pi}_H$, where $\Theta_H=[\theta_{HO},\theta_{HS}]$.

\vspace{-0.em}
\subsubsection{Training Techniques}
\label{sec:joint-training}
In a nutshell, the loss function of the high-level policy can be summarized as
\begin{equation}
\label{eq:ConsMACloss}
    \mathcal{L}_{\mathrm{high}}(\Theta_H, \phi_H, \Psi)=-J_{\mathrm{high}}(\Theta_H) - J_V(\phi_H) + \mathcal{L}_{\text{ConsMAC}}(\Psi).
\end{equation}
The optimization of $\Psi$ for ConsMAC is a standard supervised learning problem that can be solved by a gradient descent algorithm.
Then the RL-based module $\Theta_H$ utilizes the message $\textbf{m}^{(t+1)}$, yielded by ConsMAC, as part of the local information for decision-making as in \eqref{eq:PE}, and employs the gradient obtained by RL to optimize the policy.
It should be noticed that even though the message $\textbf{m}^{(t+1)}$ is used as the input of the policy module, the RL loss doesn't backpropagate to ConsMAC.
In other words, ConsMAC is designed as an independent information processing module to provide the global consensus, and the anchor point selector needs to learn how to use it through the RL method. During the training process, both ConsMAC and RL-based modules are alternately updated.
On the other hand, in our hierarchical architecture, the training of the high-level policy is based on the lower-level network, which means that specific decisions of the agents in the environment are given by the low-level policy. Therefore, it is necessary to derive how to compute the specific RL gradient of $\bm{\pi}_{\Theta_H}$ in this nested scenario. 
In this regard, the gradient of $J_{\mathrm{high}}$ will be given by the following theorem. For convenience, we omit the superscript $(t)$ in this part and simplify $\textbf{m}$ and $\textbf{z}$ as $\textbf{s}$.
Therefore, the high- and low-level policies are denoted as $\bm{\pi}_{H}(\textbf{p}_a | \textbf{s})$ and $\bm{\pi}_L(\textbf{u} | \textbf{s}, \textbf{p}_a)$, respectively.
To align with the PPO-based CI-HRL framework, we assume the joint policy for the agent to interact with the environment (i.e., $\bm{\pi}_{\Theta} = [\bm{\pi}_L, \bm{\pi}_{H}]$) is a stochastic policy, which implies that the two policies are not deterministic simultaneously.

\begin{theorem} 
\label{thm:highpolicygradient}
Given the high-level RL-based policy $\bm{\pi}_{H}$, and the low-level policy $\bm{\pi}_{L}$, the gradient of the objective function $J_{\mathrm{high}}(\Theta_H)$ with respect to the variable $\Theta_H$ is
\begin{equation}
\begin{aligned}
    \nabla_{\Theta_H}& J_{\mathrm{high}}(\Theta_H)= \mathbb{E}_{\textbf{s}, \textbf{u}\sim \bm{\pi}_L, \ \textbf{p}_a \sim \bm{\pi}_{H}}[ ( \alpha_H \nabla_{\Theta_H} \ln {\bm{\pi}}_{H}(\textbf{p}_a | \textbf{s})
     \\
    &+ \alpha_L \nabla_{\textbf{p}_a} \ln {\bm{\pi}}_L(\textbf{u} | \textbf{s}, \textbf{p}_a) \nabla_{\Theta_H} {\bm{\pi}}_{H}(\textbf{p}_a | \textbf{s})) Q(\textbf{s},\textbf{u}) ],
\end{aligned}
\end{equation}
where $Q$ is the state-action value function of ${\bm{\pi}}_L$, $\textbf{u}$ is the joint action of all agents. $\alpha_H$ is set to $1$ if the high-level policy is stochastic, and $0$ otherwise. $\alpha_L$ is associated with the low-level policy following the same binary convention.
\end{theorem}

\begin{proof} 
See Appendix \ref{sec:proof_thm:highpolicygradient}.
\end{proof}

To boost the performance of the trained high-level policy, we treat the low-level policy as a deterministic policy and use a navigation-only low-policy, which neglects the constraint of formation, to pre-train the stochastic high-level policy. Afterward, we fine-tune the high-level policy in the complete environment with the formation constraint. 
To sum up, the training procedure is summarized in Algorithm \ref{alg:high}.

\section{Experimental Results and Discussions}\label{sec:exp}
In this section, we evaluate our method in both the multi-agent particle environment (MPE) \cite{lowe2017multi} and 
an SITL simulation environment based on the ROS (Robot Operating System) system, Gazebo-Classic physics simulator \cite{koenig2004design}, and PX4 autopilot \cite{meier2015px4} for quadrotor UAVs. 
Notably, different formation patterns correspond to different formation numbers as shown in 
Fig. \ref{fig:forma}. 
\vspace{-1.em}

\begin{algorithm}[!tbp]
\caption{The Training of CI-HRL}\label{alg:high}
\begin{algorithmic}[1]

    \STATE Initialize ConsMAC, the RL-based module and value function with random parameters $\Psi$, $\Theta_H$, $\phi_H$, respectively;
    \STATE Initialize the training batch size of ConsMAC $N_{b}$, and the replay memory $\mathcal{B}_\mathrm{C}\leftarrow\varnothing$;
    \STATE Initialize the adversary $\mathfrak{A}$ and the set of target areas $\mathcal{T}$;
    \\
    \textbf{// Pre-training of high-level policy}
    \FOR{each train episode}
        \STATE Randomly initialize the environment state $\textbf{s}^{(0)}$;\\
        \textbf{// Training of RL-based module}
        \FOR{$t= \{ 1,\cdots, T\}$}
            \STATE Each agent $i$ obtains $\textbf{o}_i^{(t)}$ from $\textbf{s}^{(t)}$ and receives $\textbf{M}_\mathrm{nei,i}^{(t)}$ from neighbors;
            \STATE Calculate $\textbf{m}_i^{(t+1)}$ by \eqref{eq:AL} and $\textbf{p}_{a,i}^{(t)}$ by \eqref{eq:PE};
            \STATE Store $\langle \textbf{o}_i^{(t)}, \textbf{M}_\mathrm{nei,i}^{(t)}, \textbf{p}_{a,i}^{(t)}, \textbf{s}_i^{(t)}\rangle $ in $\mathcal{B}_\mathrm{C}$.
            \STATE Calculate the state value $V_{\phi_H}(\textbf{s}^{(t)})$;
            \STATE Calculate actions $\textbf{u}_i^{(t)}$ by a navigation-only policy;
            \STATE Execute actions, obtain the reward $\textbf{R}_\mathrm{H}^{(t)}$ and update state $\textbf{s}^{(t)} \rightarrow \textbf{s}^{(t+1)}$;
        \ENDFOR
        \STATE Update $\Theta_H$, $\phi_H$ according to \eqref{eq:ConsMACloss} while incorporating the training techniques of MAPPO;
        \\
        \textbf{// Training of ConsMAC}
        \FOR{each update epoch of ConsMAC}
            \STATE Sample a batch of $N_{b}$ tuples from $\mathcal{B}_\mathrm{C}$;
            \STATE Calculates $\hat{\textbf{g}}_i^{(t)}$ by \eqref{eq:mem}-\eqref{eq:GE};
            \STATE Update $\Psi$ according to \eqref{eq:celoss} via Adam optimizer;
        \ENDFOR
    \ENDFOR   
    \\
    \textbf{// Training of low-level policy}
    \STATE Train the low-level policy $\bm{\pi}_L$ by Algorithm \ref{alg:low};
    \\
    \textbf{// Fine-tuning }
    \STATE  Fine tuning the high-level modules $\Psi$, $\Theta_H$, $\phi_H$ in the environment by step 4-20, but the actions $\textbf{u}_i^{(t)}$ is calculated by the low-level policy $\bm{\pi}_L$.
    \ENSURE The trained $\Psi$, $\Theta_H$, $\Theta_L$;
\end{algorithmic}
\end{algorithm}

\subsection{Experimental Settings}
During the experiments, we utilize some common training techniques in MAPPO such as orthogonal initialization, gradient clipping, and value normalization, while the value functions of both the high- and low-level policies are implemented by $3$-layer MLP with a hidden size of $128$.
The training data is collected in $20$ threads simultaneously and the update epoch of PPO is $15$. Adam optimizer is used with a learning rate of $1\times 10^{-4}$.
Meanwhile, aligned with MARL-related mainstream UAVs studies for MC-PEG\cite{hu2024transfer}, \cite{Zhang2023Games}, we focus on the two-dimensional experiments with a fixed altitude, where the coordinates of the UAVs are randomly initialized in the range $[-2, 2]$ m along the $x$ and $y$ axes.
Moreover, the key parameter settings of the environment are summarized in Table \ref{tab:para_env} and the specific settings of each level are as follows.

\subsubsection{Low-level settings}
In our experiments, the low-level policy is updated for $500$ episodes, each of which has $100$ time-steps and the anchor point is randomly initialized in the range $[-5, 5]$ m for both $x$ and $y$ axes. 
The formation \& navigation module, obstacle avoidance module and the output layer in \eqref{eq:lowlevel} are all composed of $3$-layer MLP with a hidden size of $128$ for all $c\in \mathcal{C}$.
In each step of AT-M, we set up different environments as mentioned in Sec. \ref{sec:single} with obstacle densities of $0$ and $3\times 10^{-2}$/m\textsuperscript{2}, respectively, for training.
However, when the number of agents is small, direct training in an environment with obstacles already leads to satisfactory results, and recombining modules may reduce performance. Therefore, in subsequent performance comparison, we only focus on the pattern greater than $5$.

\subsubsection{High-level settings}
The high-level policy is updated for $1,000$ episodes during pre-training and $500$ episodes during fine-tuning, each of which has $400$ time-steps but anchor points are generated by the high-level policy every $10$ steps.
The predator primarily uses a single-agent PPO policy, with the same network structure as the low-level policy of UAVs, and is set to chase the nearest group with at least $3$ members and avoid obstacles.
For convenience, the discrete coordinate is used for the high-level policy, and the set of possible anchor points in $x$-axis and $y$-axis is $\{-8,0,8\}$ m, implying a $9$-dimensional high-level action space.
The distance encoder consists of a linear layer, while the global estimator is implemented by a $3$-layer MLP and the rest of the MLP-based modules have $2$ layers. 
Besides, we implement a $4$-head attention layer, and the dimension of the message $\textbf{m}$ and attention layer is $64$, while the hidden size of other MLPs is $128$. 
Moreover, after every $20$ episodes of interaction, the collected data are utilized to train ConsMAC for $5,000$ times with a batch size of $2,048$.
We assume the number of target areas is set to $2$ (i.e. $|\mathcal{T}|=2$), with each area randomly selected from the set of coordinates $\{(-8,-8),(-8,8),(8, -8),(8,8)\}$ m.

\begin{table}[tbp]
    \vspace{-1.5em}
    \centering
    \caption{The key parameter settings of the environment.}
    \label{tab:para_env}
    \begin{tabular}{c|c}
    \toprule
    \textbf{Parameters} & \textbf{Settings} \\
    \midrule
     Number of UAVs & $N = 8$ \\
     The set of possible quantities & $\mathcal{C}=\{3,4,5,6,7,8\}$ \\
     Range of velocity per UAV (m/s) & [$-1$, $1$]\\
     Range of adversary velocity (m/s)  &[$-0.75$ , $0.75$]\\
     Observation distance (m) & $\delta_{\mathrm{obs}} = 3$  \\
     Reward weights of $\textbf{R}_\mathrm{L}$ & $(\omega_\mathrm{f},\omega_\mathrm{n},\omega_\mathrm{c})= (15, 4, 100)$\\
     Reward weights of $\textbf{R}_\mathrm{H}$ & $(\omega_\mathrm{t},\omega_\mathrm{n},\omega_\mathrm{e})= (10, 0.1,  100)$\\
     PPO Hyperparameter  & $(\gamma, \varepsilon, \lambda) = (0.8, 0.2, 0.95)$ \\
    \bottomrule
    \end{tabular}
    \vspace{-1.em}
\end{table}

\subsection{Performance of Low-Level Policy}\label{sec:lowlevel}
\subsubsection{Evaluation Metrics}
We evaluate the formation performance in terms of the average \emph{reward} per time step $\textbf{R}_\mathrm{L}$, the \emph{formation stability} \textbf{F}, the \emph{navigation efficiency} \textbf{N}, and average \emph{collision probability} \textbf{C}. Specifically, the formation stability \textbf{F} counts the average time of formation maintenance per episode 
(i.e., the time when the HD-based formation error defined in Appendix \ref{sec:formationreward} is less than $1$ m), while the formation accuracy complies with the required $<1.5$ m positioning accuracy in D2D communication\cite{wang2023recent}. 
For navigation tasks, the navigation efficiency \textbf{N} quantifies the average time staying in proximity to the destination (i.e., $\|\bar{\textbf{p}}^{(t)}-\textbf{p}_a\|<1$ m), such a positional accuracy in swarms can significantly enhance payload deployment success rates\cite{vadduri2023precise}. Additionally, the average collision probability \textbf{C} evaluates obstacle avoidance and 
the average reward per time step  $\textbf{R}_\mathrm{L}$ provides a comprehensive performance evaluation aligned with operational objectives.

\begin{table}[t]
\vspace{-.5em}
\centering
\caption{Performance comparison of AT-M with baselines. }
\label{tab:atmappo}
\vspace{-.5em}
\begin{tabular}{c|c|cccc}
\toprule
c & Method 
&  $\textbf{R}_\mathrm{L}$ $\uparrow $ & \textbf{F} (s) $\uparrow$ & \textbf{N} (s) $\uparrow$ & \textbf{C} ($\%$) $\downarrow $ \\ 
\midrule

\multirow{7}{*}{$6$}  & MADDPG\cite{lowe2017multi}
& - & $58.12$ & $4.26$ & - \\
&ORCA-F\cite{sui2020formation}
& $-92.55$ & $24.52$ & $45.06$ & $0.40$ \\
&CL-M\cite{yan2022relative}
& $-90.79$  & $48.80$ & $38.94$ & $0.78$\\
&AT-M-Step 1
& $-424.45$ & $\mathbf{59.64}$ & $\mathbf{46.20}$ & $11.40$   \\
&AT-M-Step 2
& $-96.56$ & $17.7$ & $35.50$ & $0.86$ \\
&\textbf{AT-M-Step 3}
& $-96.00$ & $49.72$ & $36.66$ & $0.78$\\
&\textbf{AT-M-Safe}
& $\mathbf{-46.28}$ & $29.54$ & $24.24$ & $\mathbf{0.36}$\\

\midrule
\multirow{7}{*}{$7$}  & MADDPG\cite{lowe2017multi}
& - & $0.84$ & $\mathbf{46.94}$ & - \\
&ORCA-F\cite{sui2020formation}
& $-118.10$ & $15.74$ & $35.50$ & $0.48$ \\ 
&CL-M\cite{yan2022relative}
& $-113.69$ & $29.72$ & $15.28$ & $0.74$ \\
&AT-M-Step 1
& $-660.84$ & $\mathbf{55.78}$ & $40.26$ & $19.46$ \\ 
&AT-M-Step 2
& $-125.84$ & $8.14$ & $28.40$ & $1.08$ \\
&\textbf{AT-M-Step 3}
& $-94.32$ & $40.24$ & $24.26$  & $0.74$ \\
&\textbf{AT-M-Safe}
& $\mathbf{-55.61}$ & $23.92$ & $16.88$ & $\mathbf{0.46}$\\

\midrule
\multirow{7}{*}{$8$}  & MADDPG\cite{lowe2017multi}
& - & $0.02$ & $49.30$ & -  \\
&ORCA-F\cite{sui2020formation}
& $-195.23$ & $7.74$ & $28.52$ & $0.66$ \\ 
&CL-M\cite{yan2022relative}
& $-121.90$ & $29.30$ & $34.08$ & $1.30$  \\
&AT-M-Step 1
& $-506.48$ & $\mathbf{45.84}$ & $\mathbf{53.92}$ & $13.20$  \\ 
&AT-M-Step 2
& $-266.94$ & $0.10$ & $1.68$ & $2.56$ \\
&\textbf{AT-M-Step 3}
& $-106.94$ & $41.52$ & $36.90$ & $1.14$ \\
&\textbf{AT-M-Safe}
& $\mathbf{-72.57}$ & $20.26$ & $26.60$ & $\mathbf{0.62}$\\
\bottomrule
\end{tabular}
\vspace{-1.4em}
\end{table}
\subsubsection{Performance Comparison}\label{sec:low_compare}
To test the performance of the AT-M under more severe conditions with denser obstacles, we have increased the density of obstacles to $5\times 10^{-2}$/m\textsuperscript{2}.
We evaluate the curriculum learning-based formation \cite{yan2022relative}, denoted as CL-M, classical MADDPG\cite{lowe2017multi}, traditional optimal reciprocal collision avoidance (ORCA) based formation ORCA-F\cite{sui2020formation}, and our AT-M of each step in MPE. Table \ref{tab:atmappo} shows the average results of $50$ episodes of tests. 
To address safety-critical scenarios, we further propose AT-M-Safe, a conservative variant of AT-M. By slightly adjusting LiDAR parameters (Instead of $d_{im}^{(t)}$, $d_{im}^{(t)}-0.2$ is used in \eqref{eq:lowlevel} when the $m$-th LiDAR of the UAV detects an obstacle), UAVs initiate precautionary avoidance measures earlier.
The results of AT-M-Step 2 indicate that initializing a model in the case of dense obstacles leads to excessive collision penalties, hindering the learning of effective formation strategy, but it still maintains commendable obstacle avoidance capabilities. 
Furthermore, it can be observed that our proposed AT-M (i.e., AT-M-Step 3) outperforms other baselines in terms of overall performance balance when the number of agents increases, which 
demonstrates that introducing a well-trained obstacle avoidance model into a formation-capable model and fine-tuning the integrated model can significantly produce appealing performance. 
In addition, MADDPG \cite{lowe2017multi} is also trained in an environment without obstacles, but even so, when pattern $c$ is larger than $6$, MADDPG can no longer learn effective strategies, and its training overhead is very large compared to other methods.
Meanwhile, ORCA-F \cite{sui2020formation}, which is based on ORCA and a fully connected leader-follower framework, can greatly reduce the probability of collisions. However, ORCA-F performs poorly when integrated with downstream formation and navigation tasks, and becomes overly conservative as the number of agents increases, leading to further performance degradation, which highlights its limitations.
Meanwhile, since the collision avoidance reward in \eqref{eq:coll} penalizes proximity to obstacles for proactive avoidance, AT-M-Safe achieves superior collision avoidance than ORCA-F while maintaining formation and navigation capabilities, further validating the universality and flexibility of the AT-M methodology. 
Nevertheless, AT-M-Safe is conservative as well, with significantly degraded formation and navigation performance compared to AT-M-Step 3. Thus, it serves as a backup for dense obstacle environments, while AT-M-Step 3 suffices for subsequent experiments.
Different from AT-M, the CL-M \cite{yan2022relative} involves gradually increasing obstacle density during training.
Notably, the performance gap between AT-M and CL-M gradually increases as the number of agents increases, since the curriculum learning method gradually enhances obstacle perception and inadvertently leads to a partial and irreversible loss of formation ability.
In contrast, our AT-M can achieve obstacle avoidance while maintaining the original formation ability by fusing models from different stages and fine-tuning.
\begin{figure}[tbp]
\vspace{-1.5em}
    \centering
\subfigure[Curve of loss during policy distillation.]{
\includegraphics[width = 0.23\textwidth]{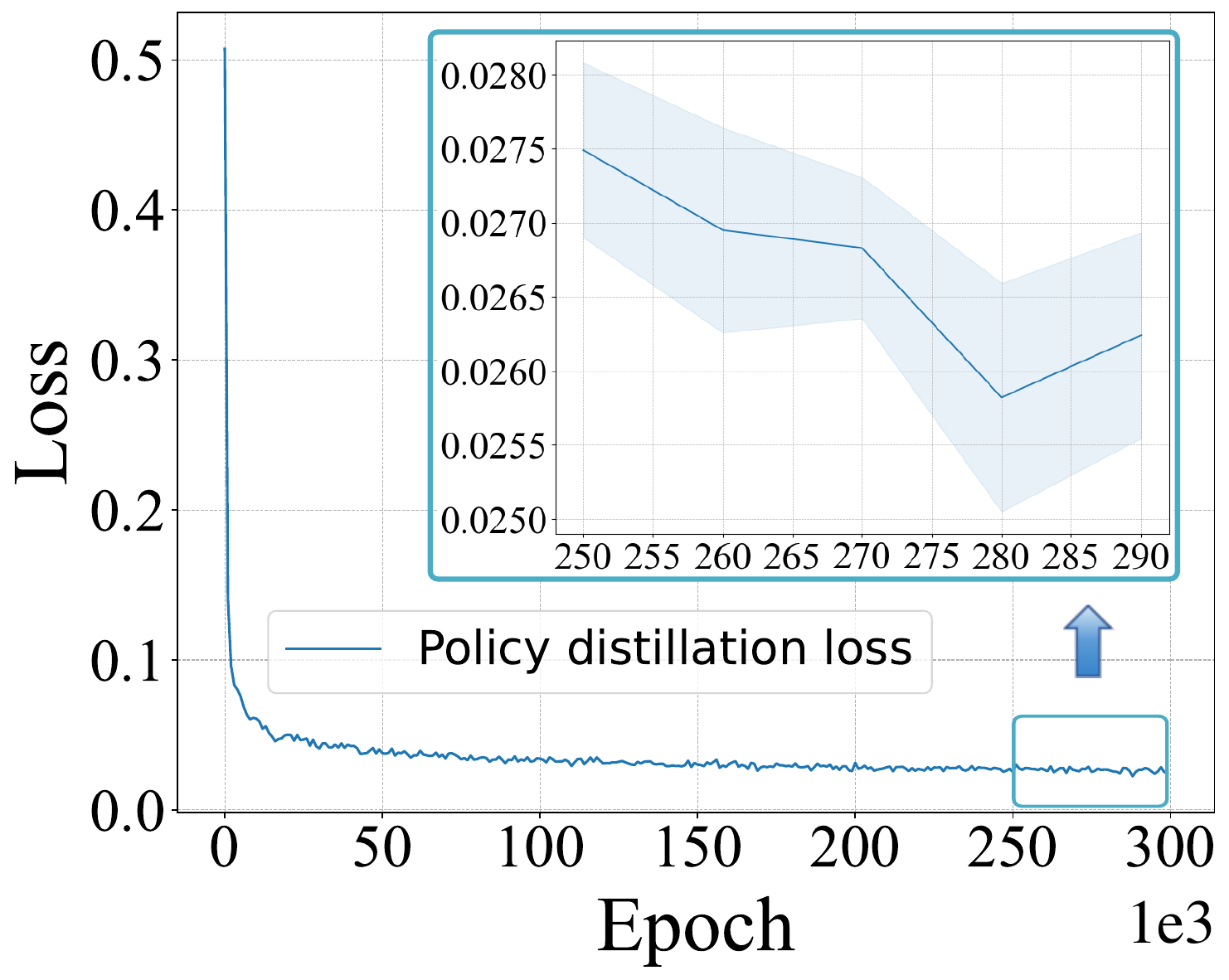}\label{fig:kd_loss}}
\subfigure[Formation reward when a random agent drops out every 200 time steps.]{
\includegraphics[width = 0.23\textwidth]{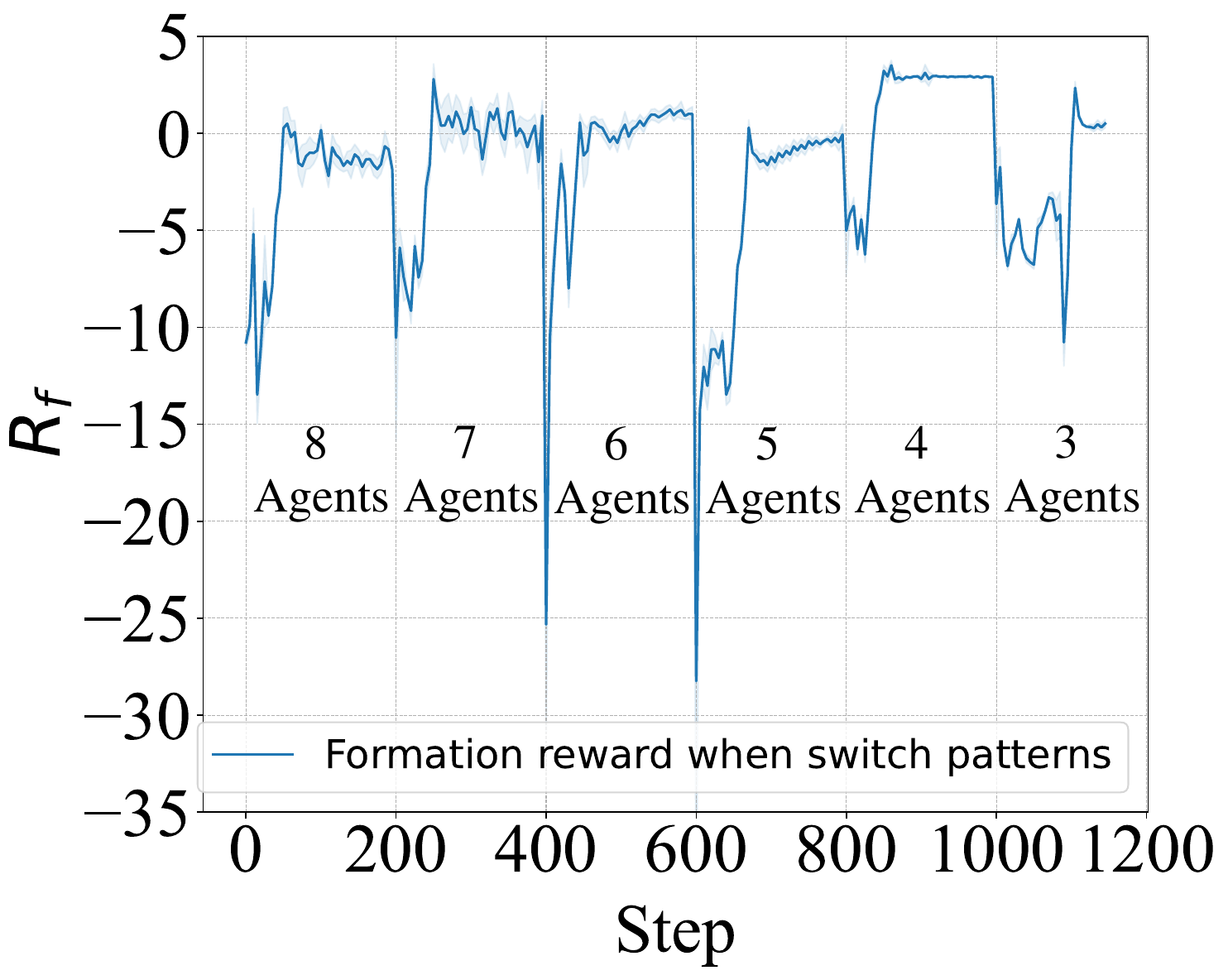}\label{fig:ad}}
\vspace{-0.5em}
\caption{Simulation results of adaptive formation.}
\vspace{-1.2em}
\end{figure} 

\subsubsection{Adaptive Formation Results}
For the adaptive formation mentioned in Sec. \ref{sec:adaptive}, we pre-train six formation models with the pattern $c$ ranging from $3$ to $8$ in Fig. \ref{fig:forma}, and store the corresponding tuples $\langle \textbf{p}_a, \textbf{z}, \textbf{u} \rangle$ for $60,000$ time-steps as the replay memory $\mathcal{B}$. Besides, the hidden size of the distilled model is $256$ and the Adam optimizer is used with a learning rate of $1\times 10^{-4}$. 
Meanwhile, the distilled student model is trained for $300,000$ episodes and the batch size of each episode of sampling is $600$. The curve of loss during policy distillation is shown in Fig. \ref{fig:kd_loss}, while Fig. \ref{fig:ad} illustrates the curve of formation reward over time. 
In particular, given the initial existence of $8$ agents, some randomly selected UAVs are assumed to be no longer observed. In the adaptive formation task, agents are supposed to perceive the change in the fleet number itself and adapt the formation policy swiftly.

\vspace{-1.em}
\subsection{Performance of ConsMAC}\label{sec:consmac}
\subsubsection{Effectiveness and Superiority of ConsMAC}
We compare our ConsMAC module with three representative MARL communication methods (i.e., the attention-based message aggregation method TarMAC \cite{das2019tarmac}, the supervised learning-based information extraction method MASIA \cite{guan2022efficient} and the SOTA communication method NVIF \cite{chai2023nvif}), 
and the communication content and overhead of each method are shown in Table \ref{tab:communication}.
In addition, we add the original MAPPO and the latest MARL SOTA algorithm HASAC \cite{liu2024maximum} without communication to reflect the effect of the communication module.
Fig. \ref{fig:consmac_compare} presents the corresponding performance comparison while Table \ref{tab:communication} lists the communication costs.
As shown in Fig. \ref{fig:consmac_compare}, ConsMAC significantly outperforms other baselines, and it is worth noting that both ConsMAC-A and ConsMAC-O, which refer to training ConsMAC by using global outputs of anchor points or global observations of all agents as labels, respectively, significantly improve the overall performance, demonstrating the robustness and generality of our method.
On the other hand, TarMAC accelerates the convergence rate by introducing the implicit communication vector, 
but its final performance is not as good as the original MAPPO, possibly due to that it does not guide the communication content and produces invalid information.
Besides, both the convergence rate and the final effect of MASIA are also reduced in the local communication environment, which indicates that only transmitting the original observation information cannot provide enough guidance for the agent to make decisions.
Notably, NVIF 
converges as fast as TarMAC and ultimately outperforms both TarMAC and MASIA, demonstrating its improvements over existing communication methods. However, it still lags behind ConsMAC and MAPPO, indicating insufficient guidance in communication content. Meanwhile, 
when the number of training iterations aligns with other baselines, HASAC merely leads to significantly inferior results.

\begin{figure}[tbp]
\vspace{-2.em}
\centering
\includegraphics[width = 0.4\textwidth]{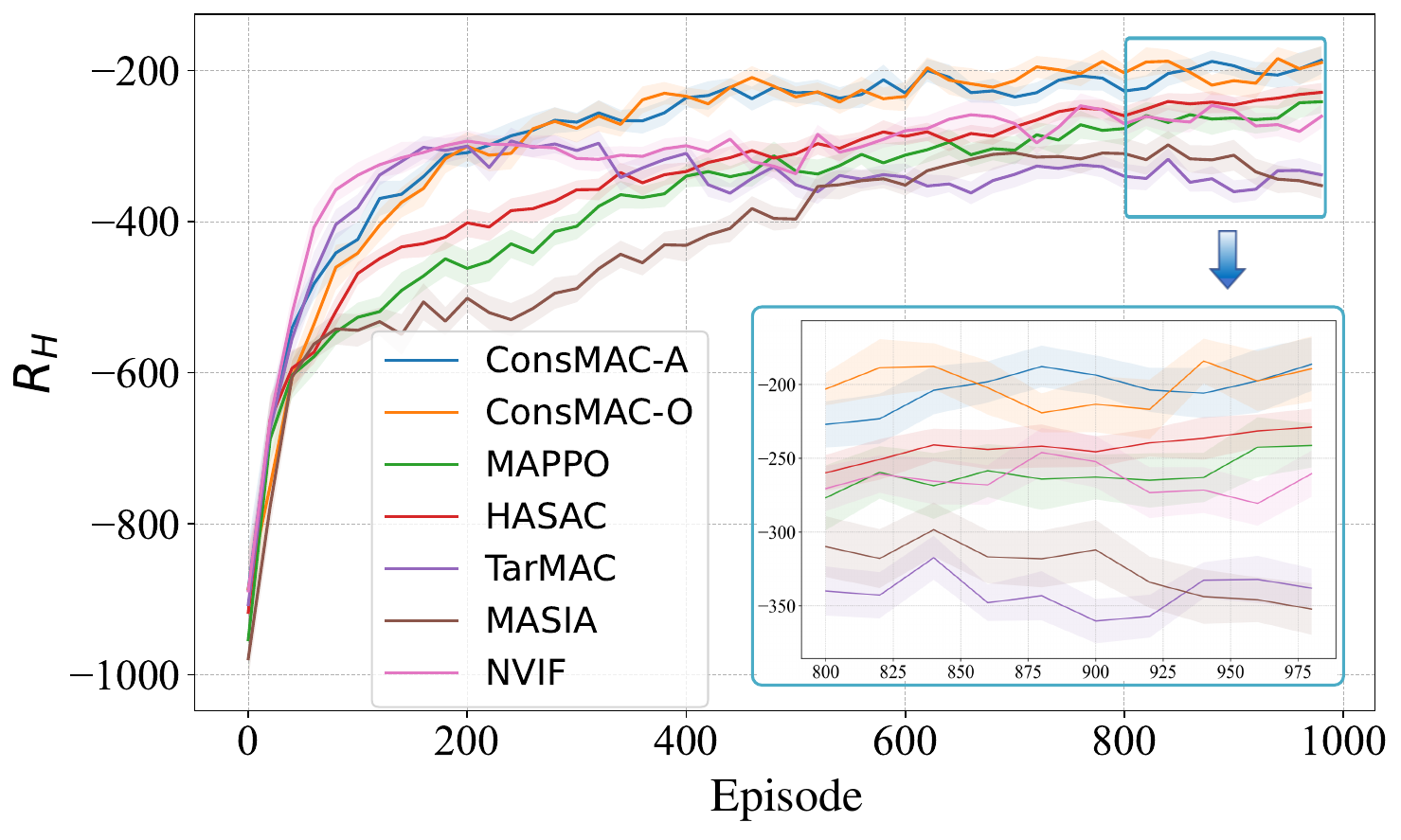}
\vspace{-.8em}
\caption{Comparison to other communication methods.}
\vspace{-.2cm}
\label{fig:consmac_compare}
\end{figure}

\begin{table}[!tbp]
    \centering
    \vspace{-0.8em}
    \caption{Communication Overhead of each method.}
    \label{tab:communication}    
    \vspace{-0.5em}
    \begin{tabular}{c|c|c} 
    \toprule
    {Method}  & Communication Content  & Dimension \\
    \midrule
    {MAPPO}\cite{yu2022surprising}& -  & $0$ \\
    HASAC\cite{liu2024maximum} & -  & $0$ \\
    {TarMAC}\cite{das2019tarmac} & Latent Vector & $64$ \\
    {MASIA}\cite{guan2022efficient}  & Local Observation & $38$ \\
    {NVIF}\cite{chai2023nvif}   & Latent Vector + Local Observation & $102$\\
    \textbf{ConsMAC (Ours)}  & Latent Vector & $64$ \\
            \bottomrule
    \end{tabular}
    \vspace{-.5em}
\end{table}

\begin{figure}[t]
\vspace{-2em}
\centering
\includegraphics[width = 0.38\textwidth]{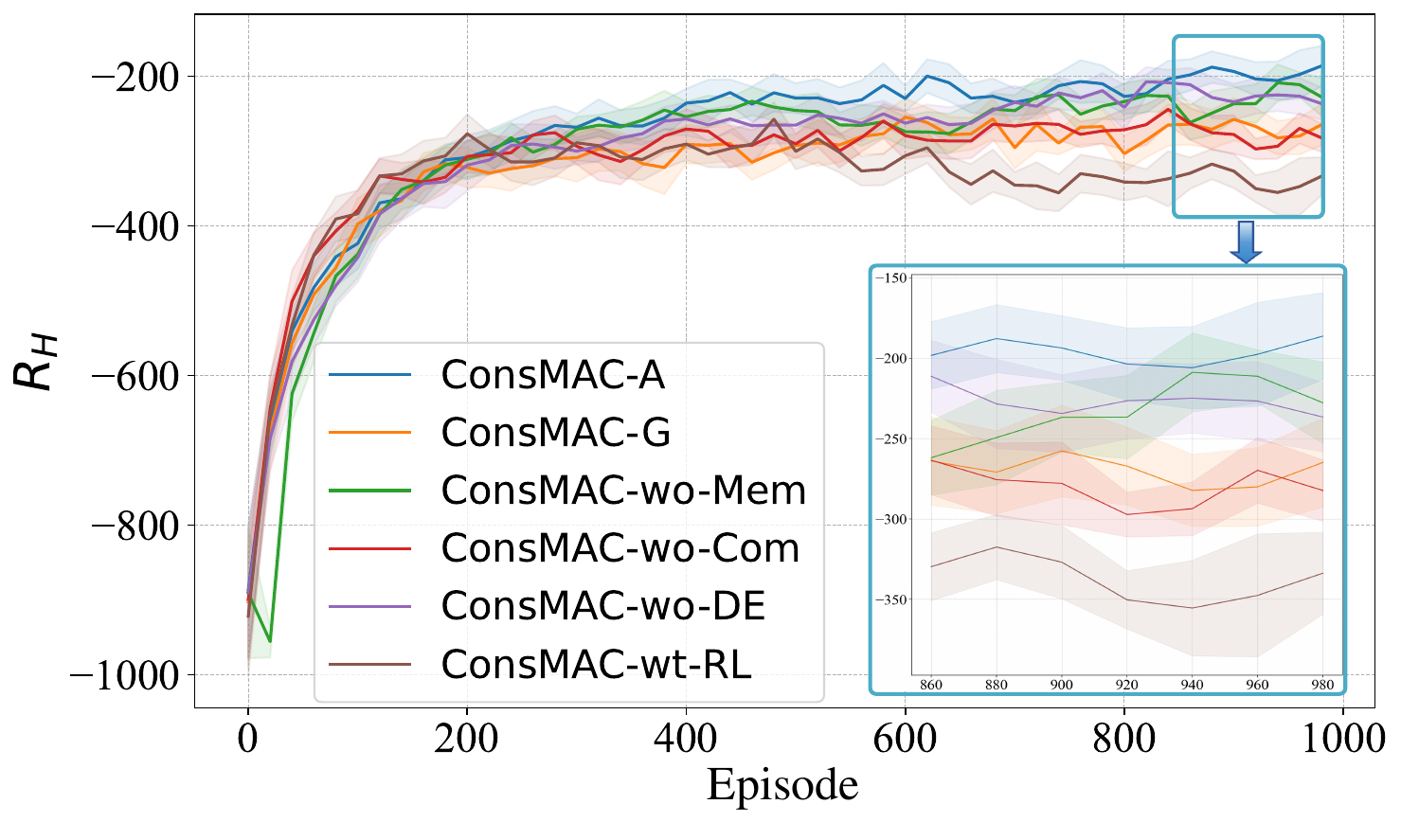}
\vspace{-1.em}
\caption{The ablation results of ConsMAC-A.}
\label{fig:consmac_ablation}
\vspace{-.2em}
\end{figure}

\subsubsection{Ablation Study of ConsMAC-A}
We perform the ablation study of ConsMAC-A to further demonstrate the effectiveness of each part in ConsMAC.
Fig. \ref{fig:consmac_ablation} illustrates the results between ConsMAC-A and some variants. ConsMAC-wo-Com denotes the ConsMAC without communicated messages, while we remove the memory module and distance encoder in ConsMAC-wo-Mem and ConsMAC-wo-DE, respectively.
Moreover, ConsMAC-wt-RL uses both supervised learning loss and RL loss to update ConsMAC.
The results in Fig. \ref{fig:consmac_ablation} verify the individual contribution of each module. Meanwhile, ConsMAC-G uses the estimated state $\hat{\textbf{g}}$ in \eqref{eq:GE} as the communicated message instead of $\textbf{m}$, and the results show that the effect of communicating the hidden layer vector is significantly better than that of communicating $\hat{\textbf{g}}$, and greatly contributes to the overall performance improvement.
Besides, it is worth noting that adding the RL loss to the ConsMAC reduces the performance, which demonstrates the importance of independent learning in ConsMAC.

\begin{figure}[!t]
\centering
\vspace{-.2em}
\includegraphics[width = 0.35\textwidth]{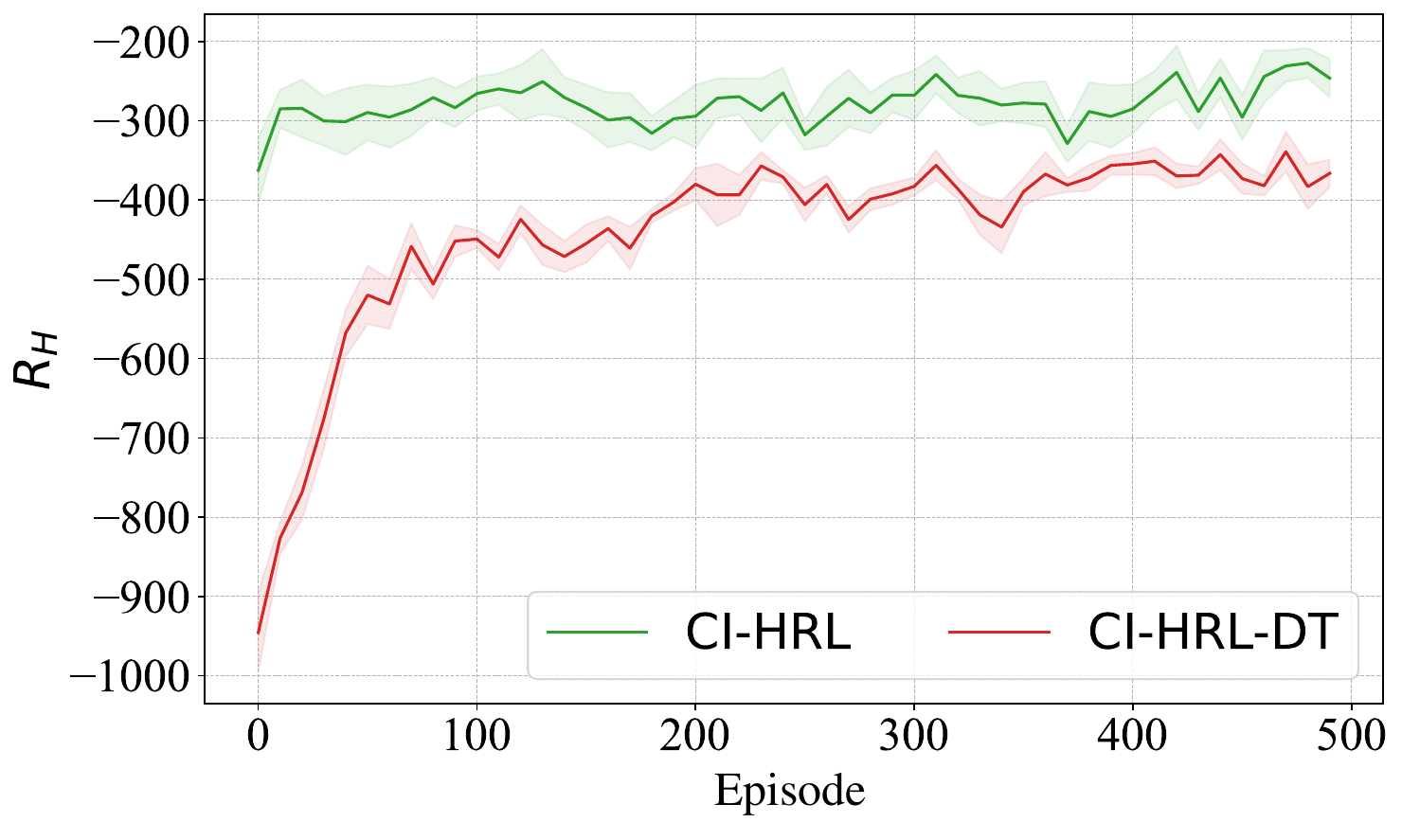}
\vspace{-0.5em}
\caption{The effectiveness of fine-tuning hierarchical policies in CI-HRL.}
\label{fig:ft}
\vspace{-0.4em}
\end{figure} 

\begin{table}[!t]
\vspace{-0.8em}
\centering
\caption{Performance comparison of CI-HRL with baselines. }
\label{tab:cihrl}
\vspace{-.5em}
\setlength{\tabcolsep}{2.2pt}
\begin{tabular}{c|ccccc}
\toprule
Method & $\textbf{R}_{\mathrm{H}}$ $\uparrow $  & $ R_{\mathrm{t}}$ $\uparrow $  & $ R_{\mathrm{n}}$ $\uparrow $ & $ R_{\mathrm{e}}$ $\uparrow $ & $ \textbf{E}$ $\downarrow $ \\ 
\midrule
MAPPO\cite{yu2022surprising}+AT-M & $-397.29$ & $54.33$ & $-334.95$ & $-116.68$ & $8.78$  \\ 
HASAC\cite{liu2024maximum}+AT-M & $-381.07$ & $105.70$ & $-347.82$ & $-138.95$  & $11.42$  \\  
TarMAC\cite{das2019tarmac}+AT-M & $-432.81$ & $29.81$ & $-367.06$ & $-95.56$ & $7.16$  \\ 
MASIA\cite{guan2022efficient}+AT-M & $-404.45$ & $31.89$ & $-336.17$ & $-100.19$ & $7.72$  \\ 
NVIF\cite{chai2023nvif} +AT-M  & $-403.79$ & $16.12$ & $-340.58$ & $-79.33$ & $5.94$ \\
\midrule
CI-HRL-w-CL-M-wo-FT  & $-408.55$ & $38.81$ & $-326.08$ & $-121.29$  & $9.48$ \\
CI-HRL-w-CL-M  & $-320.69$ & $74.13$ & $\mathbf{-288.11}$ & $-106.70$  &  $8.40$ \\
CI-HRL-DT & $-366.95$ & $79.72$ & $-381.52$ & $-65.16$ & $5.16$ \\ 
CI-HRL-wo-FT & $-351.27$ & $58.31$ & $-301.65$ & $-107.93$ & $7.66$  \\ 
\textbf{CI-HRL (Ours)} & $\mathbf{-281.56}$ & $\mathbf{107.37}$ & $-327.27$ & $\mathbf{-61.66}$ & $\mathbf{4.46}$ \\ 
\bottomrule
\end{tabular}
\vspace{-2.em}
\end{table}

\subsection{Performance of CI-HRL}\label{sec:cihrl}

\subsubsection{MPE Simulation}
\paragraph{Effectiveness and Superiority of CI-HRL}
On the basis of ConsMAC-A, we follow the Algorithm \ref{alg:high} and fine-tune the overall CI-HRL with the trained low-level policy.
Besides, we directly train the high-level policy with the low-level policy for comparison, denoted as CI-HRL-DT, while we suffix CI-HRL methods without the last fine-tuning by ``wo-FT''.
Fig. \ref{fig:ft} compares the learning curves of both fine-tuning and direct-training.
Furthermore, in Table \ref{tab:cihrl}, we evaluate each method for $50$ episodes and record the rewards for all parts, where $\textbf{E}$ denotes the average time of dangerous situations per round 
(i.e., the distance between agents and the adversary is less than $2$ m).
The related results indicate that the high-level decision-making in CI-HRL significantly outperforms other MARL algorithms, achieving larger task rewards in target areas. 
Meanwhile, due to variations in the low-level strategies, a direct concatenation of these components cannot score more effectively. Nevertheless, the pre-trained CI-HRL still performs better on average reward and navigation efficiency than direct joint training.
Moreover, the joint fine-tuning discussed in Section \ref{sec:joint-training} contributes to improving the performance. 
In addition, CL-M in Section \ref{sec:low_compare} is used as the low-level policy of CI-HRL, and the resulting model CI-HRL-w-CL-M severely underperforms its AT-M counterpart, but the joint fine-tuning improves the overall performance significantly, demonstrating an optimized high-level policy can compensate for the shortcomings of a less-performing low-level one.

\begin{table}[!tbp]
\vspace{-1.em}
\centering
\caption{Evaluation results of different adversary strategies. }
\label{tab:adversary}
\vspace{-.5em}
\begin{tabular}{c|cccc}
\toprule
 & PPO & DDPG & R-Largest & R-Nearest \\
\midrule
$R_{\mathrm{t}}$
& $107.37$ 
& $161.97$
& $151.74$ 
& $187.87 $
\\
$ R_{\mathrm{e}}$
& $-61.66$ 
& $-82.43$
& $-55.48 $
& $-79.54 $
 \\
$ \textbf{E}$
& $4.46 $ 
& $6.42$
& $4.44$ 
& $5.80$
\\
\bottomrule
\end{tabular}
\vspace{-1.em}
\end{table}

\begin{table}[!tbp]
\vspace{-0.5em}
\centering
\caption{Generalization Results of CI-HRL in large-scale groups. }
\label{tab:large}
\vspace{-.5em}
\begin{tabular}{c|ccccc}
\toprule
$N$ & $\textbf{R}_{\mathrm{H}} / N$ & $ R_{\mathrm{t}} / N $  & $ R_{\mathrm{n}} / N$ & $ R_{\mathrm{e}} / N $ & $ \textbf{E} / N $ \\ 
\midrule

$8$ &$-35.20$ &$13.42$ &$-40.91$ &$-7.71$ &$0.56$ \\ 
$9$ &$-36.80$ &$15.67$ &$-43.13$ &$-9.34$ &$0.68$ \\ 
$10$ &$-41.16$ &$10.11$ &$-43.61$ &$-7.66$ &$0.60$ \\ 
$12$ &$-42.26$ &$8.91$ &$-44.97$ &$-6.20$ &$0.49$ \\ 
$15$ &$-45.31$ &$4.09$ &$-41.46$ &$-7.95$ &$0.63$ \\ 

\bottomrule
\end{tabular}
\vspace{-1.5em}
\end{table}

\paragraph{Generalization Results}
To further validate the robustness and generalization of CI-HRL, we conduct extensive experiments under diverse adversary strategies and larger-scale scenarios respectively. 
In addition to the original PPO-based adversary, we also evaluate the performance of the CI-HRL under DDPG-driven and two different rule-based adversary policies (i.e., tracing the largest group, denoted as R-Largest, and targeting the nearest agent, denoted as R-Nearest).
Table \ref{tab:adversary} reveals the robust performance of CI-HRL against various adversary strategies. Notably, DDPG and R-Nearest strategies focus more on chasing the nearest agent rather than impeding the swarm's task completion, leading to lower evasion rewards. In other words, such an overemphasis on a single agent enhances the safety and task efficiency of the remaining agents, thereby improving the overall performance.
Furthermore, we expand the number of agents to $15$, and the agent-averaging results are shown in Table \ref{tab:large}.
Consistent with the trend in Table \ref{tab:atmappo}, along with the increase in the number of agents, the formation performance of low-level policy decreases, which subsequently impacts the overall task completion. However, the average navigation and evasion rewards remain relatively stable, indicating that the high-level policy can make superior decisions to compensate for the low-level performance degradation.
These results highlight the scalability and adaptability of our method to diverse adversarial dynamics, and validate its potential for deployment in complex environments, motivating us to investigate the practical performance of CI-HRL.

\begin{figure}[!tbp]
\vspace{-2em}
\centering
\includegraphics[width = 0.32\textwidth]{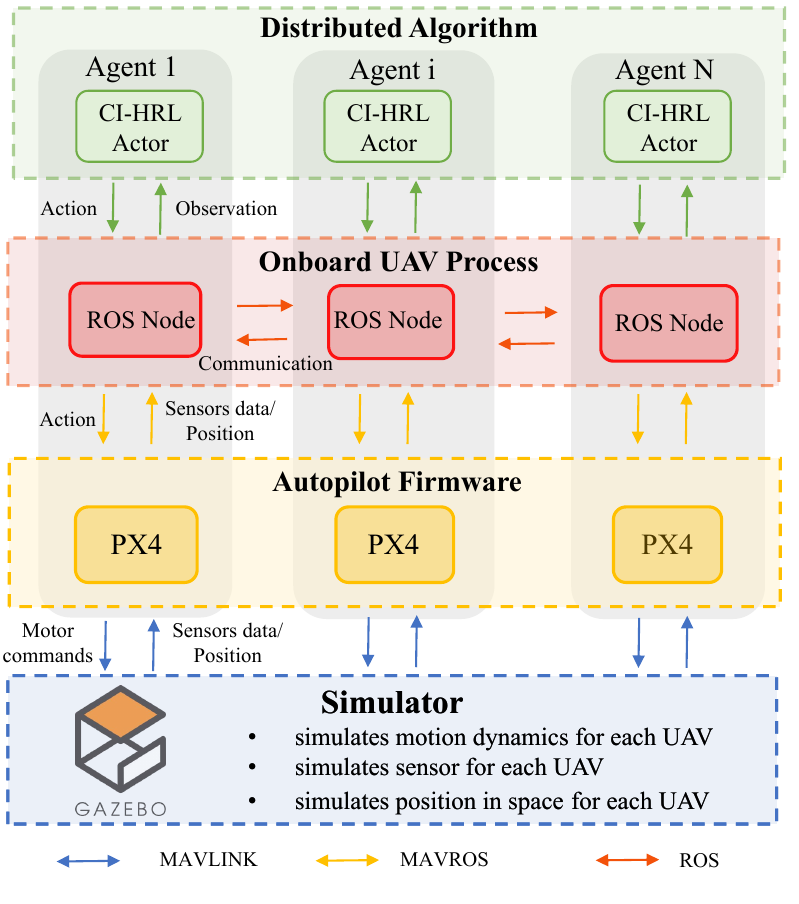}
\vspace{-1.5em}
\caption{The framework of SITL.}
\label{fig:sitl}
\vspace{-1.5em}
\end{figure} 

\begin{figure*}[!tbp]
\centering
\vspace{-1.5em}
\includegraphics[width = 0.75\textwidth]{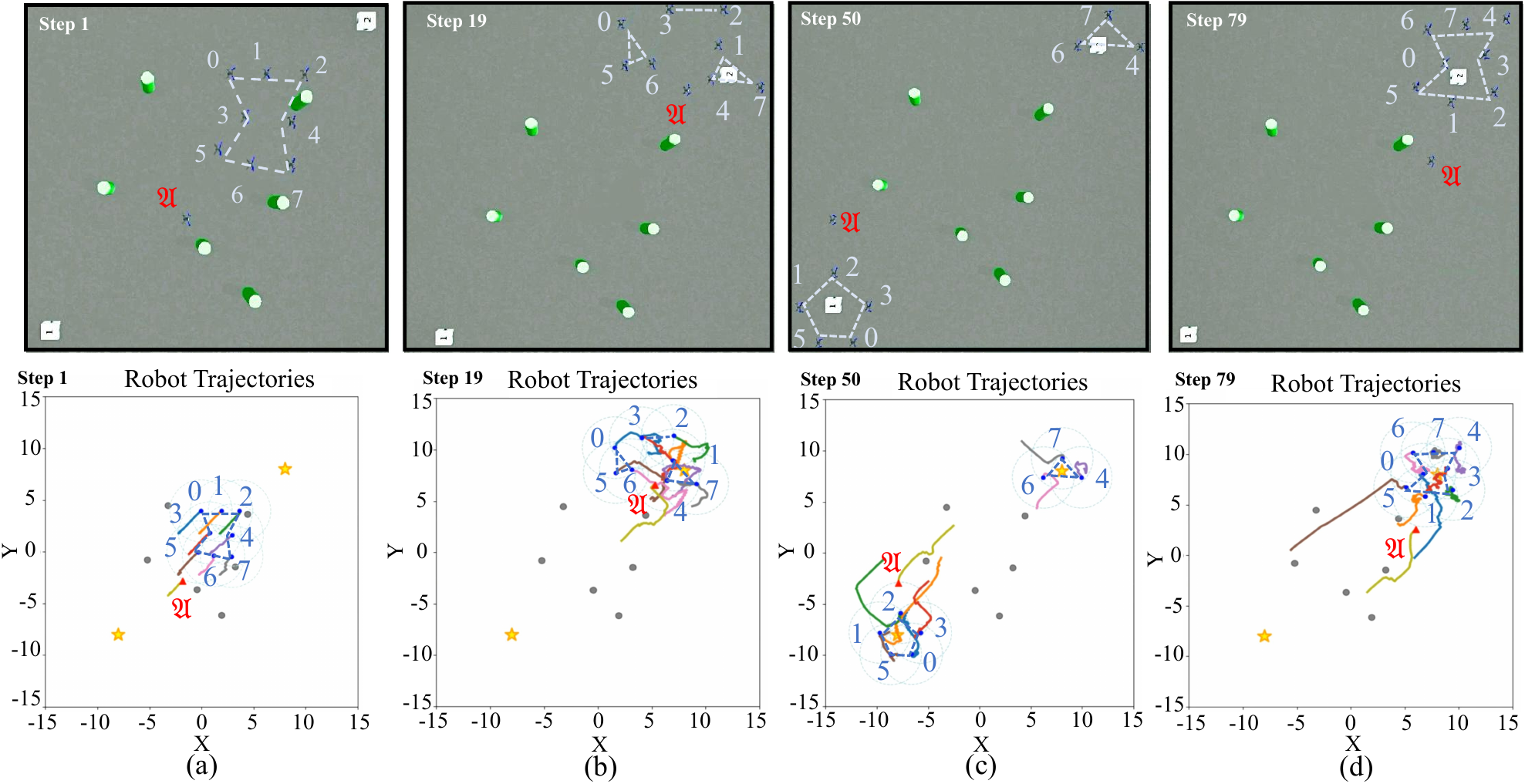}
\caption{\centering{Four typical cooperation scenarios in one episode of SITL. (a), (b), (c), and (d) stand for snapshots at high-level decision steps $1$, $19$, $50$, and $79$, respectively. }}
\label{fig:fourframe}
\vspace{-.5em}
\end{figure*} 
\begin{figure}[!t]
\centering
\includegraphics[width = 0.35\textwidth]{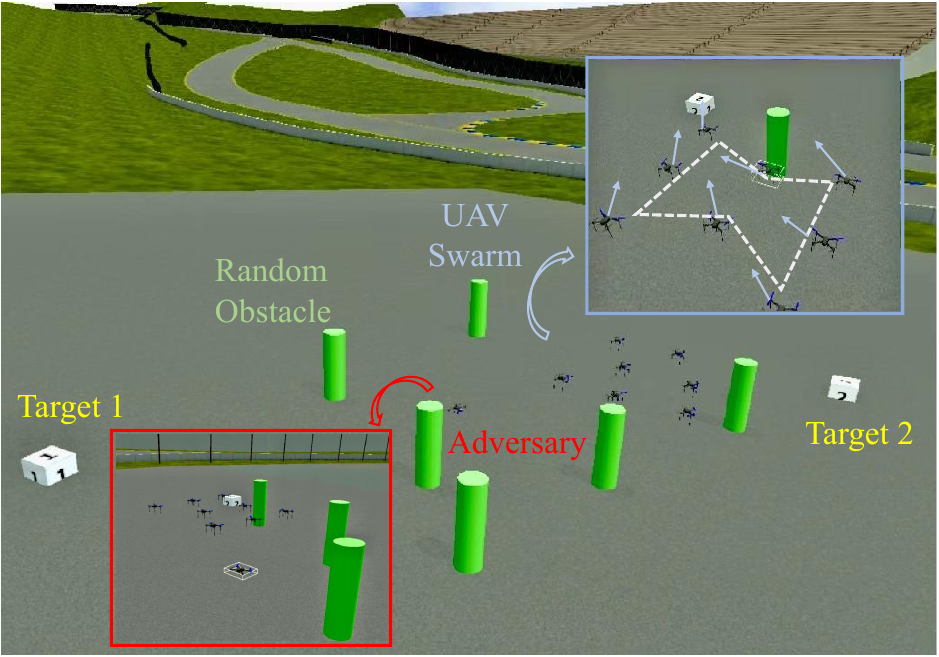}
\caption{Task overview in Gazebo Simulator for SITL.}
\label{fig:TaskOfSITL}
\vspace{-1.5em}
\end{figure}

\begin{figure*}[!ht]
\vspace{-1.5em}

\centering
\includegraphics[width = 0.77\textwidth]{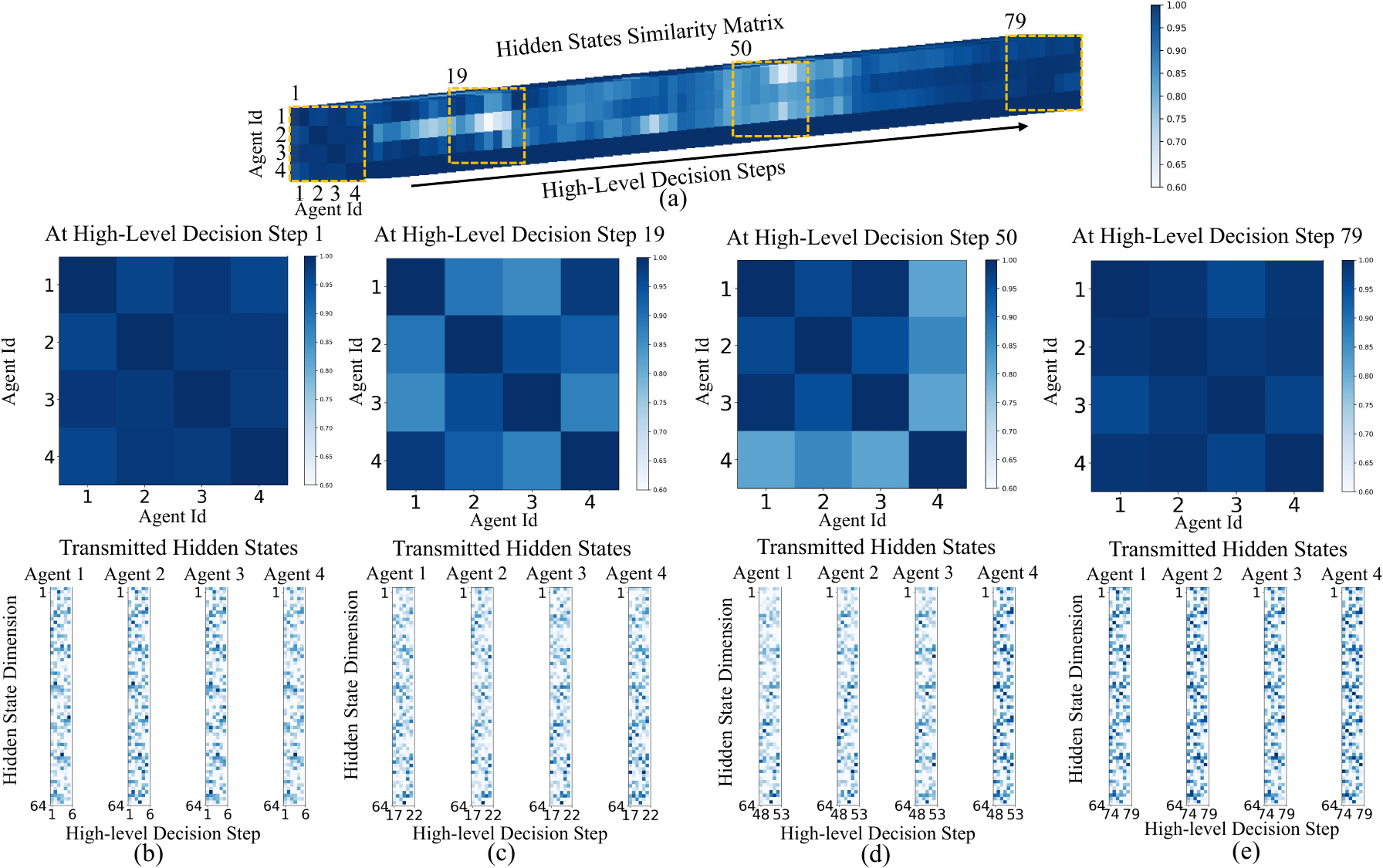}
\vspace{-.5em}
\caption{\centering{Transmitted hidden states $\textbf{m}$ for high-level decision steps and related similarity matrix analysis. (a) Hidden states similarity matrix variation for one episode. (b)$\sim$(e)  Illustration of transmitted hidden states and similarity matrices at Steps $1$, $19$, $50$, and $79$. }}
\label{fig:hiddenStates}
\vspace{-1.em}
\end{figure*} 

\subsubsection{SITL Simulation}\label{sec:sitl}
\paragraph{The Framework of SITL}
To verify the performance of the CI-HRL algorithm deployed on a quadrotor UAV in the real world, as shown in Fig. \ref{fig:sitl}, we develop a ROS-based SITL simulation environment with Gazebo-Classic physics simulator and PX4 autopilot \cite{meier2015px4} for quadrotor UAVs. Different from existing open-source SITL work such as XT-Drone \cite{xiao2020xtdrone}, in our simulation the UAV node makes fully distributed and asynchronous decisions based on partial observations within a limited communication range, and the CEFC task is simulated in high fidelity with multiple obstacles, target areas, and game scenario, as illustrated in Fig. \ref{fig:TaskOfSITL}.
Specifically, each UAV in offboard control mode corresponds to an independent Python flight control process (CI-HRL algorithm with ConsMAC-A). Moreover, based on the MAVROS protocol, each UAV subscribes to/publishes ROS topics via UDP connections to broadcast its states and observations regarding the adversary, obstacles, target areas and neighboring agents. Then, each UAV calculates the expected acceleration according to the CI-HRL output and publishes the result to the topic subscribed by PX4. 
Meanwhile, PX4 gets the UAV's real-time pose and sensor messages via a TCP connection with the Gazebo. Based on the received acceleration requirement, PX4 calculates motor and actuator values from the PID controller and sends to the Gazebo. Afterward, Gazebo determines the next frame's pose and sensor data according to the UAV's dynamic model and sends it back to PX4.

\paragraph{Strategy Analysis}
We depict four typical cooperation steps in one episode of SITL and the corresponding partial motion path based on the pose history from ROS topic in Fig. \ref{fig:fourframe}. Besides, to verify how CI-HRL works, we save the 64-dimension message $\textbf{m}_i^{(t)}$ that is transmitted to neighboring agents. Furthermore, we compute the cosine similarity of the messages from UAV $\{1, 2, 3, 4\}$ at each high-level decision step to analyze the target intention of agents, as an illustration in Fig. \ref{fig:hiddenStates}(a). Notably, higher cosine similarity of messages implies more similar intention and consistent anchor points, since the communicated messages $\textbf{m}_i^{(t)}$ imply the prediction of the global state as \eqref{eq:GE} and the selection of anchor points largely depends on the aggregated message $\textbf{m}_i^{(t + 1)}$ followed by the CI-HRL as \eqref{eq:PE}. It can be observed from Fig. \ref{fig:fourframe}(a) that at the beginning of the simulation (i.e., Step $1$), given the existence of the adversary, all agents take unanimous decisions towards the top right corner target rather than the target $\textbf{p}_1=(-8, -8)$. The similarity matrix shown in Fig. \ref{fig:hiddenStates}(b) suggests all the message similarity among four agents is higher than $0.9$, standing for agents' consensus on the target decision at the moment. At Step $19$, as shown in Fig. \ref{fig:fourframe}(b), due to the adversary's approach, the corresponding agents, which originally formed in $\Delta_8$, undergo an adaptive transformation. Specifically, Agents $\{1, 4, 7\}$ form an $\Delta_3$ group and move in directions different from Agents $\{2, 3\}$ and $\{0, 5, 6\}$, who seem to be responsible for diverting the adversary's attention, according to the similar intention of $\{1,  4\}$ and $\{2, 3\}$ in Fig. \ref{fig:hiddenStates}(c). During the adversary approaches to the left side at Step $50$, Agents $\{4,6,7\}$ reassess the safety of the target area $\textbf{p}_2=(8, 8)$ and automatically form $\Delta_3$, while Agents $\{0,1,2,3,5\}$ proceed to the anchor point $(-8,-8)$ and adaptively form a $\Delta_5$ group, as Fig. \ref{fig:fourframe}(c) shows. Correspondingly, it can be observed from Fig. \ref{fig:hiddenStates}(d) that the communicated messages of Agents $\{1, 2, 3\}$ maintain a similarity higher than $0.9$, while the similarity with Agent $4$ turns less than $0.85$. Such an observation is consistent with the case where agents in two different groups choose different anchor points. 
Later at Step $79$, with the adversary moving close to $(-8,-8)$, Agents $\{0, 1,2,3,5\}$ escape from $\textbf{p}_1$, reach the anchor point $(8,8)$ and eventually form $\Delta_8$ as shown in Fig. \ref{fig:fourframe}(d), when the messages become highly similar again as shown in Fig. \ref{fig:hiddenStates}(e).

\paragraph{Impact of Real-World Factors}
To validate the robustness of our algorithm deployed at real-world UAV systems, we simulate more practical factors (e.g. winds and sensing deviation) with Gazebo plugins.
Table \ref{tab:realistic factor} records the times of critical events such as formation, navigation, and collision probability during $1,000$-step game in SITL. 
It shows that AT-M can competently handle environmental dynamics.
For example, even confronted with strong wind, whose speed follows a Gaussian distribution with a mean of $8$ m/s, the formation and navigation completion remains approximately $78.4\%$ compared to that in a calm environment.
Similarly, AT-M sustains $80\%$ of its performance even when its input is skewed by a Gaussian distribution with a mean of $0.8$ m. It is safe and reliable until the deviation is minor than the obstacle radius which is $0.4$ m in our case.
Furthermore, to verify the effectiveness in a realistic distributed architecture, we deploy it across multiple devices in a LAN environment. For example, an NVIDIA Jetson TX2 NX is assigned to run a single onboard agent process, requiring $12.3$ ms per decision, while a GeForce RTX 4090-equipped computer is responsible for other $7$ agents, consuming $0.96$ ms per decision. Owing to the gap of computing efficiency, the average completion rate of asynchronous AT-M, denoted as SITL-M in Table \ref{tab:realistic factor}, reaches $97.1\%$ of the performance observed in simulation on a single device (denoted as SITL-S). 
As for CI-HRL, the average effective formation rate in $50$ episodes of SITL-M is equivalent to $92.7\%$ of that achieved by SITL-S.


\begin{table}[tbp]
\vspace{-0.8em}
\centering
\caption{Performance Validation of AT-M in SITL with different sensing deviation, wind speed and hardware deployment.}
\label{tab:realistic factor}
\vspace{-.5em}
\setlength{\tabcolsep}{5.8pt}
\begin{tabular}{c|ccc|ccc}
\toprule

Critical Event & \textbf{F} (s)  & \textbf{N} (s)  & \textbf{C} ($\%$)  & \textbf{F} (s)  & \textbf{N} (s)  & \textbf{C} ($\%$) 
\\ 
\midrule
$c$ & \multicolumn{3}{c|}{$5$} & \multicolumn{3}{c}{$6$} 
 \\ 
\cmidrule(r){1-1} \cmidrule(lr){2-4}\cmidrule(l){5-7} 
SITL-S
& $693$ & $722.4$ & $0 $ 
& $711$& $645$ & $0$  

\\
Wind-3m/s
&$669.3$ & $651.3$ & $0$ 
& $637$ & $644$ & $0 $
 \\
Wind-5m/s
&$643.2$ & $579.2$ & $0$   
& $586$ & $632.5$ & $0$
 \\
Wind-8m/s
& $558$ & $549$ & $0$  
& $530$ & $628$ & $0$ 
\\
Deviation-0.3m
&$683$ & $652$ & $0$ 
& $667$ & $639$ & $0$ 
 \\
Deviation-0.5m
&$647.5$ & $628.7$ & $0$  
& $616$ & $623$ & $0.33$ 
\\
Deviation-0.8m
& $545.5$ &$605.5$ & $2.95$
&$534$ &$551$ & $1.20$ 
 \\ 
\cmidrule(r){1-1} \cmidrule(lr){2-4}\cmidrule(l){5-7} 
SITL-M
& $652.8$ & $711.1$ & $0$  
& $676.1$ & $647.9$ & $0$
\\
\midrule
$c$ & \multicolumn{3}{c|}{$7$} & \multicolumn{3}{c}{$8$} 
 \\ 
\cmidrule(r){1-1} \cmidrule(lr){2-4}\cmidrule(l){5-7} 
SITL-S
& $746$& $707$ & $0$ 
& $712$ & $696.5$ & $0$  
\\
Wind-3m/s
& $582$ & $636$ & $0$
& $612.6$ & $631.3$& $0$
 \\
Wind-5m/s
& $577$ & $585$ & $0$
& $593.7$ & $615.7$ & $0$
 \\
Wind-8m/s
& $521$ & $508$& $0$  
& $524$ & $580$ & $0$  
\\
Deviation-0.3m
& $676.3$ & $703.2$ & $0$ 
&$616.3$ & $649$ & $0$ 
\\
Deviation-0.5m
& $606$ & $690.5$ & $0.60$ 
& $600.3$ & $643$ & $0.28$   
 \\ 
Deviation-0.8m
& $590$ & $593$ & $2.20$  
& $503$ & $551$ & $2.18$ 
\\
\cmidrule(r){1-1} \cmidrule(lr){2-4}\cmidrule(l){5-7} 
SITL-M
& $691.4$ & $700.8$ & $0 $
& $654$ & $705.7$ & $0 $ 

\\
\bottomrule
\end{tabular}
\vspace{-1.4em}
\end{table}

\section{Conclusion and Future Work}\label{sec:final}
In this work, towards accomplishing the CEFC task, we have proposed and validated a consensus inference-based hierarchical multi-agent reinforcement learning framework (CI-HRL).
Specifically, in low-level policy, we have implemented an alternative training MAPPO (AT-M) to satisfy multiple coupled constraints and better balance the trade-off between formation and navigation performance through dense random obstacles. Moreover, policy distillation has been adopted to achieve a more flexible adaptive formation.
Meanwhile, in high-level control, to infer global information from the partially observed local state and implicitly establish the consensus, ConsMAC, a consensus-oriented multi-agent communication methodology, has been designed in a novel supervised learning manner.
Finally, we have conducted extensive experiments in MPE and software-in-the-loop environments and successfully demonstrated the effectiveness and robustness of individual modules in CI-HRL. Moreover, the superiority of CI-HRL has been thoroughly validated.

Although CI-HRL has brilliant performance, there are still policy convergence challenges in super-large scale scenarios.
This limitation stems from the exponential growth of complexity for the centralized critic in the CTDE framework and the strict topological constraints in formation control. Potential solutions include fully decentralized training \cite{ma2024efficient}
or fractal-based hierarchical partitioning \cite{wu2024an}.
Meanwhile, we will carry out intense studies to further improve the stability of MARL framework, optimize communication overhead and develop methods to counter stronger adversaries. For example, to enhance UAV maneuverability in a 3D environment, we will study the incorporation of an altitude planner for altitude adjustments while maintaining swarm’s fundamental 2D behaviors.
Furthermore, recent advances in adversarial RL \cite{schott2024robust, yuan2023Robust} provide promising directions for designing adaptive adversaries that can dynamically adjust their strategies during training.  
Integrating such methods could enhance the robustness of CI-HRL against sophisticated adversarial tactics. 
\vspace{-.3cm}
\appendices
\begin{figure*}[!t]
\vspace{-2em}
\begin{equation}
\tag{24}
\label{eq:ss}
\begin{aligned}
    \nabla_{\Theta_H} J_{\mathrm{high}}(\Theta_H) &=\nabla_{\Theta_H}\sum_{\textbf{s}\in\mathcal{S}} P_{0}(\textbf{s}) V(\textbf{s}) 
    =\sum_{\textbf{s}\in\mathcal{S}} P_{0}(\textbf{s})\sum_{\textbf{s}^{\prime}\in\mathcal{S}}\sum_{t=0}^{T} \gamma^t P\left(\textbf{s} \rightarrow \textbf{s}^{\prime}, t, \bm{\pi}_{\Theta} \right) \sum_{\textbf{u}\in\mathcal{U}} \nabla_{\Theta_H} F(\textbf{u} | \textbf{s}^{\prime}) Q(\textbf{s}^{\prime}, \textbf{u})\\
    & = \sum_{\textbf{s}^{\prime}\in\mathcal{S}} \rho(\textbf{s}^{\prime})  \sum_{\textbf{u}\in\mathcal{U}} \nabla_{\Theta_H} F(\textbf{u} | \textbf{s}^{\prime}) Q(\textbf{s}^{\prime}, \textbf{u}) 
    = \sum_{\textbf{s}^{\prime}\in\mathcal{S}} \rho(\textbf{s}^{\prime})  \sum_{\textbf{u}\in\mathcal{U}} \sum_{\textbf{p}_a\in\mathcal{P}_a} \nabla_{\Theta_H} \left( \bm{\pi}_L(\textbf{u} | \textbf{s}^{\prime}, \textbf{p}_a) {\bm{\pi}}_{H}(\textbf{p}_a | \textbf{s}^{\prime}) \right) Q(\textbf{s}^{\prime}, \textbf{u})\\
    & = \sum_{\textbf{s}^{\prime}\in\mathcal{S}} \sum_{\textbf{u}\in\mathcal{U}} \sum_{\textbf{p}_a\in\mathcal{P}_a} \rho(\textbf{s}^{\prime})  \bm{\pi}_L(\textbf{u} | \textbf{s}^{\prime}, \textbf{p}_a) {\bm{\pi}}_{H}(\textbf{p}_a | \textbf{s}^{\prime}) \nabla_{\Theta_H} \ln \left({\bm{\pi}}_L(\textbf{u} | \textbf{s}^{\prime}, \textbf{p}_a){\bm{\pi}}_{H}(\textbf{p}_a | \textbf{s}^{\prime}) \right) Q(\textbf{s}^{\prime}, \textbf{u})\\
    & = \mathbb{E}_{\textbf{s}, \textbf{u}\sim \bm{\pi}_L, \textbf{p}_a \sim \bm{\pi}_{H}}\left[ \left(\nabla_{\textbf{p}_a} \ln {\bm{\pi}}_L(\textbf{u} | \textbf{s}, \textbf{p}_a) \nabla_{\Theta_H} {\bm{\pi}}_{H}(\textbf{p}_a | \textbf{s})+ \nabla_{\Theta_H} \ln {\bm{\pi}}_{H}(\textbf{p}_a | \textbf{s})  \right) Q(\textbf{s},\textbf{u}) \right]
\end{aligned}
\end{equation}
\hrulefill
\vspace{-.4cm}
\end{figure*}
\section{Details of the Reward Function Design}
\label{sec:specificreward}
In this part, we introduce the details of the reward function in Sec. \ref{sec:task_model}. 
For simplicity of representation, we denote the relative positions of agents in the group $k\in \{1,\cdots, \chi^{(t)}\}$ as $\textbf{P}^{(t)}_{k}=\{\textbf{p}_j^{(t)}-\bar{\textbf{p}}^{(t)}_k|\forall j \in \mathcal{N}_{k}\}$ where $\bar{\textbf{p}}^{(t)}_k$ is the center of the group and $\Delta_{n_k}$ is the target formation corresponding to $k$.

\subsubsection{Formation reward}\label{sec:formationreward}
Towards implementing leader-free formation control and enhancing the robustness, we adopt the HD \cite{pan2022flexible} to measure the formation error between the current and the expected formation. The HD between two topologies $\mathcal{E}_1$ and $\mathcal{E}_2$ is defined as $\mathrm{\HD}(\mathcal{E}_1,\mathcal{E}_2)=\max_{\mathbf{x}\in \mathcal{E}_1} \min_{\mathbf{y}\in \mathcal{E}_2}\|\mathbf{x}-\mathbf{y}\|$. The formation reward for each group can be obtained as
\begin{equation}
    \label{eq:rform}
    R_\mathrm{f}^{(t)} = -\sum\nolimits_{k=1}^{\chi^{(t)}} \mathrm{\HD}(\Delta_{n_k}, \textbf{P}^{(t)}_{k})  - \omega_l R_\mathrm{f}^{(t-1)},
\end{equation}
where $R_\mathrm{f}^{(0)}=0$ and $\omega_l$ is the formation lag coefficient to better reflect the formation trend.
    
\subsubsection{Navigation reward}\label{sec:nav_reward}
The navigation reward simply uses an urgency-weighted Euclidean distance between agents and targets, which can be given by
\begin{equation}
\label{eq:nareward}
    R_\mathrm{n}^{(t)}=-\sum\nolimits_{i\in \mathcal{N}} \sum\nolimits_{\mathfrak{T} \in \mathcal{T}}\kappa_{\mathfrak{T}}^{(t)}\|\textbf{p}_{i\to \mathfrak{T}}^{(t)}\|.
\end{equation}

\subsubsection{Task accomplishment reward}
The task accomplishment reward is awarded when each group $k$ reaches the target area and maintains the formation, and can be written as
\begin{equation}\label{eq:taskreward}R_\mathrm{t}^{(t)}  = \sum\nolimits_{\mathfrak{T} \in \mathcal{T}} \kappa_{\mathfrak{T}}^{(t)} \sum\nolimits_{k^*} \text{TR}_{k^*\to \mathfrak{T}}^{(t)},
\end{equation}
where
\begin{equation} \text{TR}_{k^*\to \mathfrak{T}}^{(t)} =
\begin{cases}
    n_{k^*} ,&\text{ if } \|\bar{\textbf{p}}^{(t)}_{k^*\to \mathfrak{T}}\|<\delta_{\mathrm{task}},\\
    0,&\text{otherwise}, 
\end{cases}
\end{equation}
with $\delta_{\mathrm{task}}$ denoting the radius of the target area, and $k^*\in \{1,\cdots, \chi^{(t)}\}$ indicates a group meeting a formation tolerance $\delta_{\mathrm{for}}$, i.e., $\mathrm{\HD}(\Delta_{n_{k^*}}, \textbf{P}^{(t)}_{k^*}) < \delta_{\mathrm{for}}$. 
Besides, the urgency factor $\kappa^{(t)}_{\mathfrak{T}}$ will gradually decay when the agents stay in target area $\mathfrak{T}$ as follows 
\begin{equation}
   \kappa^{(t+1)}_{\mathfrak{T}} = \max \left ( \kappa^{(t)}_{\mathfrak{T}} - \omega_d  \sum\nolimits_{k^*} \text{TR}_{k^*\to \mathfrak{T}}^{(t)}, 0 \right ),
\end{equation}
where $\kappa^{(0)}_{\mathfrak{T}}=1$ and $\omega_d$ is the decay factor.

\subsubsection{Evasion reward}
The evasion reward is designed to prevent the adversary, and can be formulated as,
\begin{equation}
    R_\mathrm{e}^{(t)}=-\sum\nolimits_{i\in \mathcal{N}} \max \left ( \delta_{\text{a,e}} - \| \textbf{p}_{i \to \mathfrak{A}}^{(t)}\|, 0 \right ) ,
\end{equation}
where $\delta_{\text{a,e}}$ is the alert distance of evasion.

\begin{table}[tbp]
    \vspace{-0.8em}
    \centering
    \caption{The parameter settings of the reward.}
    \vspace{-0.5em}
    \label{tab:reward}
    \begin{tabular}{c|c}
    \toprule
    \textbf{Parameters} & \textbf{Settings} \\
    \midrule
     Formation lag coefficient $\omega_l$ & $0.3$\\
     Decay factor $\omega_d$ & $0.003$\\
     Alert distance $\delta_{\text{a,e}}$, $\delta_{\text{a,c}}$ & ($2$ m, $0.5$ m)\\
     Formation tolerance $\delta_{\mathrm{for}}$ & $1$ m\\
     The radius of the target area $\delta_{\mathrm{task}}$ & $3$ m\\
     Minimum safety distance $\delta_{\mathrm{s}}$ & $0.2$ m \\
     Collision constants $\omega_{cr_1}, \omega_{cr_2}, C_1, C_2$ & ($24, 8, 3, 1$) \\
    \bottomrule
    \end{tabular}
    \vspace{-1.em}
\end{table}

\subsubsection{Collision avoidance reward}
The collision avoidance reward, which is designed to prevent obstacles, can be formulated as a summation of individual rewards, namely
\begin{equation}\label{eq:coll}R_\mathrm{c}^{(t)}=-\sum\nolimits_{i\in \mathcal{N}}\sum\nolimits_{j \ne i, j\in \mathcal{I}} \text{CR}_{i\to j}^{(t)},
\end{equation}
where $\mathcal{I}$ is the set of agents and obstacles, and for the minimum safety distance $\delta_{\text{s}}$ and collision alert distance $\delta_{\text{a,c}}$, 
\begin{equation} \label{eq:collsion_reward} \text{CR}_{i\to j}^{(t)} =
\begin{cases}
 \omega_{cr_1} (\delta _{\text{s}}-d_{i,m}^{(t)}) + C_1, & d_{i,m}^{(t)}<\delta _{\text{s}},\\
 \omega_{cr_2} (\delta_{\text{a,c}}-d_{i,m}^{(t)}) + C_2, & \delta _{\text{s}} < d_{i,m }^{(t)}<\delta_{\text{a,c}},\\
 0, & d_{i,m}^{(t)}>\delta _{a},
\end{cases} \\
\end{equation}
with $d_{i,m}^{(t)}=\min (d_{i1}^{(t)},...,d_{iM}^{(t)})$ related to LiDAR and positive constants $\omega_{cr_1}$, $\omega_{cr_2}$, $C_1$, $C_2$.

Furthermore, we have summarized the parameter settings of the reward in Table \ref{tab:reward}.

\section{The Proof of Theorem \ref{thm:highpolicygradient}}
\label{sec:proof_thm:highpolicygradient}

\begin{proof} 
The proof follows the policy gradient theorem\cite{sutton1999policy} and is similar to the bi-level optimization in reward shaping\cite{hu2020learning}. Based on the stochasticity of $\bm{\pi}_{H}$ and $\bm{\pi}_{L}$, there exist three cases. When both $\bm{\pi}_{H}$ and $\bm{\pi}_{L}$ are stochastic, let $F(\textbf{u} | \textbf{s}) = \sum_{\textbf{p}_a\in\mathcal{P}_a} \bm{\pi}_L(\textbf{u} | \textbf{s}, \textbf{p}_a) {\bm{\pi}}_{H}(\textbf{p}_a | \textbf{s}) $ denote the probability of taking action $\textbf{u}$ in state $\textbf{s}$.
At each time step $t$, the state-value function and the state-action value function can be written as
\begin{align}
&V(\textbf{s})=\sum_{\textbf{u}\in\mathcal{U}} F(\textbf{u} | \textbf{s}) Q(\textbf{s}, \textbf{u}),\nonumber\\
&Q(\textbf{s}, \textbf{u}) =  \textbf{R}_\mathrm{H}(\textbf{s}, \textbf{u}) +\gamma \sum_{\textbf{s}^{\prime}} P\left(\textbf{s}^{\prime}|\textbf{s}, \textbf{u}\right) V (\textbf{s}^{\prime}).\nonumber
\end{align}
Therefore, according to the policy gradient theorem\cite{sutton1999policy}, the gradient with respect to the high-level policy can be calculated by expanding $Q(\textbf{s}, \textbf{u})$,
\begin{align}
&\nabla_{\Theta_H} V(\textbf{s}) = \nabla_{\Theta_H}\left(\sum_{\textbf{u}\in\mathcal{U}} F(\textbf{u} | \textbf{s}) Q(\textbf{s}, \textbf{u}) \right)  \label{eq:valuegradient}\\
&=\sum_{\textbf{s}^{\prime}\in\mathcal{S}}\sum_{t=0}^{T} \gamma^t P\left(\textbf{s} \rightarrow \textbf{s}^{\prime}, t, \bm{\pi}_{\Theta}\right) \sum_{\textbf{u}\in\mathcal{U}} \nabla_{\Theta_H} F(\textbf{u} | \textbf{s}^{\prime}) Q(\textbf{s}^{\prime}, \textbf{u}), \nonumber
\end{align}
where  $P\left(\textbf{s} \rightarrow \textbf{s}^{\prime}, t, \bm{\pi}_{\Theta}\right)$  is the probability that state $\textbf{s}^{\prime}$ is visited after $t$ steps from state $ \textbf{s}$ under the joint policy  $\bm{\pi}_{\Theta}$ . 
Let $\rho(\textbf{s}^{\prime}) = \sum_{\textbf{s}\in\mathcal{S}} P_{0}(\textbf{s})\sum_{t=0}^{T} \gamma^t P\left(\textbf{s} \rightarrow \textbf{s}^{\prime}, t, \bm{\pi}_{\Theta}\right)$ be the discounted state distribution, where $P_{0}$ denotes the probability distribution of initial states. By using the log-derivative trick and the chain rule, we can obtain \eqref{eq:ss}. Then, the conclusion comes. 

When either $\bm{\pi}_{H}$ or $\bm{\pi}_{L}$ are stochastic, the proofs can be similar and omitted for brevity.
Thus, we prove the theorem.
\end{proof}
\bibliographystyle{IEEEtran}
\bibliography{reference}

\begin{thebibliography}{10}
\providecommand{\url}[1]{#1}
\csname url@samestyle\endcsname
\providecommand{\newblock}{\relax}
\providecommand{\bibinfo}[2]{#2}
\providecommand{\BIBentrySTDinterwordspacing}{\spaceskip=0pt\relax}
\providecommand{\BIBentryALTinterwordstretchfactor}{4}
\providecommand{\BIBentryALTinterwordspacing}{\spaceskip=\fontdimen2\font plus
\BIBentryALTinterwordstretchfactor\fontdimen3\font minus \fontdimen4\font\relax}
\providecommand{\BIBforeignlanguage}[2]{{%
\expandafter\ifx\csname l@#1\endcsname\relax
\typeout{** WARNING: IEEEtran.bst: No hyphenation pattern has been}%
\typeout{** loaded for the language `#1'. Using the pattern for}%
\typeout{** the default language instead.}%
\else
\language=\csname l@#1\endcsname
\fi
#2}}
\providecommand{\BIBdecl}{\relax}
\BIBdecl

\bibitem{Marwan2024Meta}
M.~Dhuheir\emph{,~et~al.}, ``Meta reinforcement learning for strategic {IoT} deployments coverage in disaster-response {UAV} swarms,'' in \emph{IEEE Glob. Comm. Conf.}, Kuala Lumpur, Malaysia, 2023, pp. 6159--6164.

\bibitem{Pham2017Distributed}
H.~X. Pham\emph{,~et~al.}, ``A distributed control framework for a team of unmanned aerial vehicles for dynamic wildfire tracking,'' in \emph{IEEE/RSJ Int. Conf. Intell. Robots Syst.}, Vancouver, Canada, 2017, pp. 6648--6653.

\bibitem{hu2024transfer}
P.~Hu\emph{,~et~al.}, ``Transfer reinforcement learning for multi-agent pursuit-evasion differential game with obstacles in a continuous environment,'' \emph{Asian J. Cont.}, vol.~26, pp. 2125--2140, 2024.

\bibitem{Young2020Consensus}
Z.~Young\emph{,~et~al.}, ``Consensus, cooperative learning, and flocking for multiagent predator avoidance,'' \emph{J. Adv. Robot. Syst.}, vol.~17, no.~5, pp. 1--19, Sep. 2020.

\bibitem{Vibhav2022Multi}
V.~Kedege\emph{,~et~al.}, ``Multi robot surveillance and planning in limited communication environments,'' in \emph{Int. Conf. Age. Artif. Intell.}, online, 2022, pp. 139--147.

\bibitem{konda2020decentralized}
R.~Konda\emph{,~et~al.}, ``Decentralized function approximated {Q}-learning in multi-robot systems for predator avoidance,'' \emph{IEEE Robot. Autom. Lett.}, vol.~5, no.~4, pp. 6342--6349, 2020.

\bibitem{Zhang2023RealCity}
Z.~Zhang\emph{,~et~al.}, ``A pursuit-evasion game on a real-city virtual simulation platform based on multi-agent reinforcement learning,'' in \emph{Chinese Cont. Conf.}, Tianjin, China, 2023, pp. 6018--6023.

\bibitem{yuan2025multiagent}
Z.~Yuan\emph{,~et~al.}, ``Multiagent formation control and dynamic obstacle avoidance based on deep reinforcement learning,'' \emph{IEEE Trans. Ind. Inform.}, vol. early access, pp. 1--11, 2025.

\bibitem{lowe2017multi}
R.~Lowe\emph{,~et~al.}, ``Multi-agent actor-critic for mixed cooperative-competitive environments,'' \emph{Adv. Neural Inf. Process. Syst.}, vol.~30, pp. 1--16, 2017.

\bibitem{yan2023collision}
C.~Yan\emph{,~et~al.}, ``Collision-avoiding flocking with multiple fixed-wing {UAVs} in obstacle-cluttered environments: A task-specific curriculum-based {MADRL} approach,'' \emph{IEEE Trans. Neural Netw. Learn. Syst.}, vol.~35, no.~8, pp. 10\,894--10\,908, 2023.

\bibitem{chai2023nvif}
J.~Chai\emph{,~et~al.}, ``{NVIF}: Neighboring variational information flow for cooperative large-scale multiagent reinforcement learning,'' \emph{IEEE Trans. Neural Netw. Learn. Syst.}, vol.~35, no.~12, pp. 17\,829--17\,841, 2024.

\bibitem{rashid2020monotonic}
T.~Rashid\emph{,~et~al.}, ``Monotonic value function factorisation for deep multi-agent reinforcement learning,'' \emph{J. Mach. Learn. Res.}, vol.~21, no.~1, pp. 7234--7284, 2020.

\bibitem{guan2022efficient}
C.~Guan\emph{,~et~al.}, ``Efficient multi-agent communication via self-supervised information aggregation,'' \emph{Adv. Neural Inf. Process. Syst.}, vol.~35, pp. 1--13, Nov 2022.

\bibitem{liu2024maximum}
J.~Liu\emph{,~et~al.}, ``Maximum entropy heterogeneous-agent reinforcement learning,'' in \emph{Proc. Int. Conf. Learn. Repre.}, Vienna, Austria, May 2024, pp. 1--12.

\bibitem{schulman2017proximal}
J.~Schulman\emph{,~et~al.}, ``Proximal policy optimization algorithms,'' \emph{arXiv:1707.06347}, 2017.

\bibitem{yu2022surprising}
C.~Yu\emph{,~et~al.}, ``The surprising effectiveness of ppo in cooperative multi-agent games,'' \emph{Adv. Neural Inf. Process. Syst.}, vol.~35, pp. 24\,611--24\,624, Nov 2022.

\bibitem{xie2024multi}
Y.~Xie\emph{,~et~al.}, ``Multi-{UAV} behavior-based formation with static and dynamic obstacles avoidance via reinforcement learning,'' \emph{arXiv:2410.18495}, 2024.

\bibitem{yan2022relative}
Y.~Yan\emph{,~et~al.}, ``Relative distributed formation and obstacle avoidance with multi-agent reinforcement learning,'' in \emph{Proc. IEEE Int. Conf. Robot. Automat.}, Philadelphia, PA, USA, May 2022, pp. 1661--1667.

\bibitem{jin2021hierarchical}
Y.~Jin\emph{,~et~al.}, ``Hierarchical and stable multiagent reinforcement learning for cooperative navigation control,'' \emph{IEEE Trans. Neural Netw. Learn. Syst.}, vol.~34, no.~1, pp. 90--103, 2021.

\bibitem{lu2023self}
Z.~Lu\emph{,~et~al.}, ``Self-critical alternate learning based semantic broadcast communication,'' \emph{arXiv:2312.01423}, 2023.

\bibitem{zhu2022survey}
C.~Zhu\emph{,~et~al.}, ``A survey of multi-agent deep reinforcement learning with communication,'' in \emph{Proc. 23rd Int. Conf. Auton. Agents Multiagent Syst.}, Auckland, New Zealand, 2024, p. 2845–2847.

\bibitem{amirkhani2022consensus}
A.~Amirkhani\emph{,~et~al.}, ``Consensus in multi-agent systems: {A} review,'' \emph{Artif. Intell. Review}, vol.~55, no.~5, pp. 3897--3935, Jun 2022.

\bibitem{yang2021data}
X.~Yang\emph{,~et~al.}, ``Data-based optimal consensus control for multiagent systems with policy gradient reinforcement learning,'' \emph{IEEE Trans. Neural Netw. Learn. Syst.}, vol.~33, no.~8, pp. 3872--3883, 2021.

\bibitem{sukhbaatar2016learning}
S.~Sukhbaatar\emph{,~et~al.}, ``Learning multiagent communication with backpropagation,'' \emph{Adv. Neural Inf. Process. Syst.}, vol.~29, p. 2252–2260, 2016.

\bibitem{das2019tarmac}
A.~Das\emph{,~et~al.}, ``{TarMAC}: Targeted multi-agent communication,'' in \emph{Proc. 36th Int. Conf. Mach. Learn.}, California, USA, Jun 2019, pp. 1--9.

\bibitem{kim2020communication}
W.~Kim\emph{,~et~al.}, ``Communication in multi-agent reinforcement learning: Intention sharing,'' in \emph{Proc. Int. Conf. Learn. Repre.}, online, Oct 2020, pp. 1--15.

\bibitem{wang2021tomc}
Y.~Wang\emph{,~et~al.}, ``{ToM2C}: Target-oriented multi-agent communication and cooperation with theory of mind,'' in \emph{Proc. Int. Conf. Learn. Repre.}, onine, Apr 2022, pp. 1--17.

\bibitem{xu2023consensus}
Z.~Xu\emph{,~et~al.}, ``Consensus learning for cooperative multi-agent reinforcement learning,'' in \emph{Proc. AAAI Conf. Artif. Intell.}, vol.~37, no.~10, Washington, USA, 2023, pp. 11\,726--11\,734.

\bibitem{pateria2021hierarchical}
S.~Pateria\emph{,~et~al.}, ``Hierarchical reinforcement learning: A comprehensive survey,'' \emph{ACM Computing Surveys}, vol.~54, no.~5, pp. 1--35, 2021.

\bibitem{peng2022ase}
X.~B. Peng\emph{,~et~al.}, ``Ase: Large-scale reusable adversarial skill embeddings for physically simulated characters,'' \emph{ACM Trans. Graph.}, vol.~41, no.~4, pp. 1--17, 2022.

\bibitem{xu2022autonomous}
G.~Xu\emph{,~et~al.}, ``Autonomous obstacle avoidance and target tracking of {UAV} based on deep reinforcement learning,'' \emph{J. Intell. Robot. Syst.}, vol. 104, no.~4, pp. 60--79, 2022.

\bibitem{Zhang2023Games}
R.~Zhang\emph{,~et~al.}, ``Game of drones: Multi-{UAV} pursuit-evasion game with online motion planning by deep reinforcement learning,'' \emph{IEEE Trans. Neural Netw. Learn. Syst.}, vol.~34, no.~10, pp. 7900--7909, 2023.

\bibitem{Yang2023LargScale}
H.~Yang\emph{,~et~al.}, ``Large scale pursuit-evasion under collision avoidance using deep reinforcement learning,'' in \emph{IEEE/RSJ Int. Conf. Intell. Robots Syst.}, Detroit, USA, 2023, pp. 2232--2239.

\bibitem{deng2020multi}
Z.~Deng\emph{,~et~al.}, ``Multi-agent cooperative pursuit-defense strategy against one single attacker,'' \emph{IEEE Robot. Autom. Lett.}, vol.~5, no.~4, pp. 5772--5778, 2020.

\bibitem{dayan1992feudal}
P.~Dayan\emph{,~et~al.}, ``Feudal reinforcement learning,'' \emph{Adv. Neural Inf. Process. Syst.}, vol.~5, pp. 1--8, 1992.

\bibitem{dietterich2000hierarchical}
T.~G. Dietterich, ``Hierarchical reinforcement learning with the maxq value function decomposition,'' \emph{J. Artif. Intell. Res.}, vol.~13, pp. 227--303, 2000.

\bibitem{bacon2017option}
P.-L. Bacon\emph{,~et~al.}, ``The option-critic architecture,'' in \emph{Proc. AAAI Conf. Artif. Intell.}, vol.~31, 2017, p. 1726–1734.

\bibitem{levy2017learning}
A.~Levy\emph{,~et~al.}, ``Learning multi-level hierarchies with hindsight,'' \emph{arXiv:1712.00948}, 2017.

\bibitem{meier2015px4}
L.~Meier\emph{,~et~al.}, ``Px4: A node-based multithreaded open source robotics framework for deeply embedded platforms,'' in \emph{Proc. IEEE Int. Conf. Robot. Automat.}, Seattle, USA, 2015, pp. 6235--6240.

\bibitem{bresciani2008modelling}
T.~Bresciani, ``Modelling, identification and control of a quadrotor helicopter,'' \emph{MSc theses}, 2008.

\bibitem{Yan2023Target}
P.~Yan\emph{,~et~al.}, ``Long-term tracking of evasive urban target based on intention inference and deep reinforcement learning,'' \emph{IEEE Trans. Neural Netw. Learn. Syst.}, vol.~35, no.~11, pp. 16\,886--16\,900, 2024.

\bibitem{pan2022flexible}
C.~Pan\emph{,~et~al.}, ``Flexible formation control using hausdorff distance: A multi-agent reinforcement learning approach,'' in \emph{Euro. Sign. Proc. Conf.}, Glasgow, UK, Aug 2022, pp. 972--976.

\bibitem{schulman2015high}
J.~Schulman\emph{,~et~al.}, ``High-dimensional continuous control using generalized advantage estimation,'' in \emph{Proc. Int. Conf. Learn. Repre.}, May 2016, pp. 1--17.

\bibitem{rusu2015policy}
A.~A. Rusu\emph{,~et~al.}, ``Policy distillation,'' in \emph{Proc. Int. Conf. Learn. Repre.}, San Juan, Puerto Rico, May 2016, pp. 1--13.

\bibitem{vaswani2017attention}
A.~Vaswani\emph{,~et~al.}, ``Attention is all you need,'' \emph{Adv. Neural Inf. Process. Syst.}, vol.~30, pp. 1--15, Dec 2017.

\bibitem{koenig2004design}
N.~Koenig\emph{,~et~al.}, ``Design and use paradigms for gazebo, an open-source multi-robot simulator,'' in \emph{IEEE/RSJ Int. Conf. Intell. Robots Syst.}, vol.~3, Sendai, Japan, 2004, pp. 2149--2154 vol.3.

\bibitem{wang2023recent}
Y.~Wang\emph{,~et~al.}, ``Recent progress on 3gpp 5g positioning,'' in \emph{IEEE Veh. Tech. Conf.}, 2023, pp. 1--6.

\bibitem{vadduri2023precise}
A.~Vadduri\emph{,~et~al.}, ``Precise payload delivery via unmanned aerial vehicles: An approach using object detection algorithms,'' \emph{arXiv preprint arXiv:2310.06329}, 2023.

\bibitem{sui2020formation}
Z.~Sui\emph{,~et~al.}, ``Formation control with collision avoidance through deep reinforcement learning using model-guided demonstration,'' \emph{IEEE Trans. Neural Netw. Learn. Syst.}, vol.~32, no.~6, pp. 2358--2372, 2020.

\bibitem{xiao2020xtdrone}
K.~Xiao\emph{,~et~al.}, ``Xtdrone: A customizable multi-rotor {UAV}s simulation platform,'' in \emph{Int. Conf. Robot. Autom. Sci.}, online, 2020, pp. 55--61.

\bibitem{ma2024efficient}
C.~Ma\emph{,~et~al.}, ``Efficient and scalable reinforcement learning for large-scale network control,'' \emph{Nat. Mach. Intell.}, vol.~6, no.~9, pp. 1006--1020, 2024.

\bibitem{wu2024an}
J.~Wu\emph{,~et~al.}, ``An attention mechanism and adaptive accuracy triple-dependent maddpg formation control method for hybrid uavs,'' \emph{IEEE Trans. Intell. Trans. Syst.}, vol.~25, no.~9, pp. 11\,648--11\,663, 2024.

\bibitem{schott2024robust}
L.~Schott\emph{,~et~al.}, ``Robust deep reinforcement learning through adversarial attacks and training: A survey,'' \emph{arXiv:2403.00420}, 2024.

\bibitem{yuan2023Robust}
L.~Yuan\emph{,~et~al.}, ``Robust multi-agent coordination via evolutionary generation of auxiliary adversarial attackers,'' in \emph{Proc. AAAI Conf. Artif. Intell.}, vol.~37, no.~10, Washington, USA, 2023, pp. 11\,753--11\,762.

\bibitem{sutton1999policy}
R.~S. Sutton\emph{,~et~al.}, ``Policy gradient methods for reinforcement learning with function approximation,'' \emph{Adv. Neural Inf. Process. Syst.}, vol.~12, pp. 1--7, 1999.

\bibitem{hu2020learning}
Y.~Hu\emph{,~et~al.}, ``Learning to utilize shaping rewards: A new approach of reward shaping,'' \emph{Adv. Neural Inf. Process. Syst.}, vol.~33, pp. 15\,931--15\,941, 2020.

\end{thebibliography}

\end{document}